\documentclass{ecai} 
\usepackage{latexsym}
\usepackage{amssymb}
\usepackage{amsmath}
\usepackage{amsthm}
\usepackage{booktabs}
\usepackage{enumitem}
\usepackage{graphicx}
\usepackage{color}

\newtheorem{theorem}{Theorem}
\newtheorem{lemma}[theorem]{Lemma}

\newtheorem{definition}{Definition}

\usepackage[utf8]{inputenc} % allow utf-8 input
\usepackage[T1]{fontenc}    % use 8-bit T1 fonts
\usepackage{hyperref}       % hyperlinks
\usepackage{url}            % simple URL typesetting
\usepackage{booktabs}       % professional-quality tables
\usepackage{amsfonts}       % blackboard math symbols
\usepackage{nicefrac}       % compact symbols for 1/2, etc.
\usepackage{microtype}      % microtypography
\usepackage{xcolor}         % colors
\definecolor{myBlue}{RGB}{66, 148, 247}
\definecolor{myRed}{RGB}{209, 53, 43}
\usepackage{colortbl} 
\usepackage{amsthm}
\theoremstyle{plain}

\usepackage{wrapfig}
\usepackage{subfigure}

\usepackage{mwe}
\usepackage{algorithm,algpseudocode}
\usepackage{makecell}
\usepackage{multirow}
\usepackage{color,soul}
\usepackage{mathtools}

\newcommand{\BibTeX}{B\kern-.05em{\sc i\kern-.025em b}\kern-.08em\TeX}

\begin{document}

%%%%%%%%%%%%%%%%%%%%%%%%%%%%%%%%%%%%%%%%%%%%%%%%%%%%%%%%%%%%%%%%%%%%%%%%

\begin{frontmatter}

%%% Use this command to specify your submission number.
%%% In doubleblind mode, it will be printed on the first page.

\paperid{8744} 

%%% Use this command to specify the title of your paper.

\title{Machine Unlearning for Streaming Forgetting}

%%% Use this combinations of commands to specify all authors of your 
%%% paper. Use \fnms{} and \snm{} to indicate everyone's first names 
%%% and surname. This will help the publisher with indexing the 
%%% proceedings. Please use a reasonable approximation in case your 
%%% name does not neatly split into "first names" and "surname".
%%% Specifying your ORCID digital identifier is optional. 
%%% Use the \thanks{} command to indicate one or more corresponding 
%%% authors and their email address(es). If so desired, you can specify
%%% author contributions using the \footnote{} command.

\author[A]{\fnms{Shaofei}~\snm{Shen}}
\author[A]{\fnms{Chenhao}~\snm{Zhang}}
\author[A]{\fnms{Yawen}~\snm{Zhao}} 
\author[A]{\fnms{Alina}~\snm{Bialkowski}} 
\author[B]{\fnms{Weitong}~\snm{Chen}} 
\author[A]{\fnms{Miao}~\snm{Xu}\thanks{Corresponding author}} 

\address[A]{The University of Queensland} \address[B]{University of Adelaide}

%%% Use this environment to include an abstract of your paper.

\begin{abstract}
Machine unlearning aims to remove knowledge of the specific training data in a well-trained model. Currently, machine unlearning methods typically handle all forgetting data in a single batch, removing the corresponding knowledge all at once upon request. However, in practical scenarios, requests for data removal often arise in a streaming manner rather than in a single batch, leading to reduced efficiency and effectiveness in existing methods. Such challenges of streaming forgetting have not been the focus of much research. In this paper, to address the challenges of performance maintenance, efficiency, and data access brought about by streaming unlearning requests, we introduce a streaming unlearning paradigm, formalizing the unlearning as a distribution shift problem. We then estimate the altered distribution and propose a novel streaming unlearning algorithm to achieve efficient streaming forgetting without requiring access to the original training data. Theoretical analyses confirm an $O(\sqrt{T} + V_T)$ error bound on the streaming unlearning regret, where $V_T$ represents the cumulative total variation in the optimal solution over $T$ learning rounds. This theoretical guarantee is achieved under mild conditions without the strong restriction of convex loss function. Experiments across various models and datasets validate the performance of our proposed method. 
\end{abstract}

\end{frontmatter}

%%%%%%%%%%%%%%%%%%%%%%%%%%%%%%%%%%%%%%%%%%%%%%%%%%%%%%%%%%%%%%%%%%%%%%%%

\section{Introduction}

Machine unlearning aims at safeguarding the privacy of individuals concerning sensitive and private data~\cite{CCPA,GDPR}. The objective of machine unlearning is to remove information associated with a selected group of data, referred to as \emph{forgetting data}, from a well-trained model while retaining the knowledge encapsulated in the \emph{remaining data}. Presently, research in this field has made progress in designing effective unlearning algorithms. Current unlearning methods typically consider the forgetting data as a single batch and approach unlearning as a singular adjustment process. This process removes all information from the forgetting data at once, then uses the remaining data to repair and update the model, preserving its functionality~\cite{boundary,scrub}.

In practical scenarios, data removal requests from sensitive information owners are usually made incrementally rather than in predetermined batches. For example, social media users might request the deletion of recommendations learned from their personal browsing history at any time, resulting in a continuous stream of individual requests. This streaming nature means that requests are submitted immediately as users identify their needs and often arise in a streaming manner rather than being grouped and processed together. To address the streaming forgetting problem, where removal requests occur incrementally or in streams, existing batch unlearning approaches often handle each request from scratch, one by one. Therefore, the problems with batch unlearning are further amplified in multiple rounds of unlearning in streaming settings.

The first issue is the accumulated performance drop. Although the performance degradation usually happens in machine unlearning, previous methods have tried to reduce the degradation in the batch unlearning~\cite {sisa,boundary,ts,Amnesiac,laf,unroll,acml1}. However, this degradation can accumulate across multiple rounds, leading to a significant decline in overall model performance over time. The second issue is the low efficiency, which stems from the repeated performance repairs on the remaining data~\cite {sisa,boundary,ts,laf}. Due to the high overlap of remaining data across different rounds, there are significant time and computational costs associated with reprocessing the same data, making the unlearning process inefficient. Thirdly, frequent access to the remaining data can be problematic due to data regularization policies~\cite{CCPA,GDPR}. In many cases, parts of the training data may no longer be accessible or may be subject to strict access controls~\cite{boundary,zero-shot,RememberWhatYouWanttoForget,aaai1,geniu}, hindering the necessary updates and repairs to the model. 

These limitations emphasize the need for a novel streaming forgetting method that can handle long sequences of data removal requests without performance degradation~\cite{ada,camu}. It should optimize time and memory consumption~\cite{arxivsurvey,fastyet} with minimal reliance on the training data. Although some prior works have explored streaming forgetting in the context of ensemble models~\cite{continualunlearning}, such an approach predominantly focuses on the forgetting of the candidate model rather than the forgetting of training data. Only two studies have directly tackled the problem of streaming data forgetting~\cite{onlinelr,onlineforgetvision}, both of which impose specific constraints on model structures to realize unlearning, like vision transformer~\cite{onlineforgetvision} or linear model~\cite{onlinelr}. As a result, the problem of streaming unlearning for data instances from general well-trained models remains an open research challenge.

In this paper, we propose a novel streaming unlearning method that addresses the accumulated drops in both effectiveness and efficiency while reducing the need for frequent access to training data. To estimate the unlearning risk without training data, we formalize unlearning as a distribution shift problem. The shifted distribution caused by removing forgetting data serves as prior knowledge to make the unlearning process more efficient and accurate. To incrementally update the model towards unlearning, we propose a risk estimator to achieve the optimal model in each streaming round and propose the corresponding streaming unlearning approach -- SAFE (\textbf{S}tream-\textbf{A}ware \textbf{F}org\textbf{e}tting). Our approach departs from traditional batch unlearning by incorporating dynamic regret risk and reducing reliance on original training data. Furthermore, our theoretical analysis guarantees the effectiveness of SAFE by providing the upper bound on the unlearning regret risk of $O(\sqrt{T} + V_T)$, where $V_T$ represents the cumulative total variation in the optimal solution over $T$ learning rounds. This result holds without assuming the convexity of the loss function. To evaluate the practical performance of SAFE, we conduct empirical experiments on different deep neural networks across various datasets.

The contributions of this paper can be summarized as follows:
\vspace{-1mm}
\begin{itemize}
\setlength\itemsep{-0mm}
    \item We introduce the streaming unlearning paradigm and the SAFE algorithm to address the streaming unlearning problem. SAFE maintains high predictive performance on the remaining data while ensuring high unlearning efficiency. Notably, it does not require repeated access to the original training data during unlearning.
    \item We are the first to provide an $O(\sqrt{T}+ V_T)$ upper bound on the unlearning regret risk of the proposed algorithms through theoretical analysis, which is much lower than the error of previous methods.
    \item Through empirical evaluations across multiple datasets and models, we demonstrate that SAFE achieves higher or comparable performance more efficiently than other batch and streaming unlearning methods, especially on neural network-based models. 
\end{itemize}

\section{Background and Problem Setup}\label{sec:prelibg}

In this section, we first introduce the standard machine unlearning problem, including both exact and approximate formulations. Building upon these foundations, we then formulate the problem in the streaming unlearning setting.

\subsection{Background}\label{sec:preli}

\paragraph{Exact Unlearning.}
We consider a supervised learning setting where a model is trained on a dataset $D$ consisting of input-label pairs $(\mathbf{x}, y)$, where $\mathbf{x} \in \mathcal{X}$ represents the input features and $y \in \mathcal{Y}$ represents the corresponding labels. 
Let $\mathsf{A}$ be a learning algorithm that outputs a model $f(\cdot; w_0)$ given the data $D$, i.e., $f(\cdot; w_0) = \mathsf{A}(D)$, where $w_0$ denotes the model parameters and $f$ belongs to a hypothesis space $\mathcal{H}$. 
Suppose that a subset $F \subset D$ is later identified as data that must be forgotten. 
The goal of \textit{machine unlearning} is to remove the influence of $F$ from the trained model. Classical \textit{Exact unlearning} retrains the model from scratch on the remaining data $D \setminus F$ (i.e., $D$ excluding $F$), yielding the unlearned model $f(\cdot; w^*)$:
\begin{align}\label{eq:w_estar}
w^* = \arg\min_w \mathcal{L}(D \setminus F, w),
\end{align}
where $\mathcal{L}$ is the task-specific training loss. Specifically, $\mathcal{L}(D \setminus F, w) = \sum_{(\mathbf{x}, y) \in D \setminus F} \ell(f(\mathbf{x}; w), y)$, where $\ell(\cdot, \cdot)$ denotes the prediction loss function (e.g., cross-entropy for classification).

\paragraph{Approximate Unlearning.}
Exact unlearning requires retraining from scratch on the remaining data $D \setminus F$, which is often computationally prohibitive or infeasible in practice~\cite{sisa,ada}. 
To address this, approximate unlearning seeks to efficiently produce a model whose behavior closely approximates that of a retrained model without $F$. Ideally, this corresponds to finding model parameters $w$ that minimize the behavioral difference between the model predictions $f(\cdot; w)$ and those of the retrained model $f(\cdot; w^*)$ across the entire dataset $D$:
\begin{align}\label{eq:unlearning-whole-obj}
\arg\min_w \;  \sum_{(\mathbf{x}, y) \in D} \mathrm{diff}( f(\mathbf{x}; w), f(\mathbf{x}; w^*) ),
\end{align}
where $\mathrm{diff}(\cdot, \cdot)$ denotes a suitable discrepancy measure between model predictions. Unlike exact unlearning, this optimization seeks to approximate the behavior of $f(\cdot; w^*)$ without retraining the model from scratch. In practice, this objective is commonly decomposed into two components that separately address retention and forgetting, balanced by a trade-off parameter $\lambda$~\cite{ts,salun}:
\begin{align}\nonumber
\widetilde{w} = \arg\min_w \Bigg(
& \sum_{(\mathbf{x}, y) \in D \setminus F} \mathrm{diff}\big(f(\mathbf{x}; w), f(\mathbf{x}; w^*)\big) \\\label{eq:unlearningobj}
&+ \lambda \sum_{(\mathbf{x}, y) \in F} \mathrm{diff}\big(f(\mathbf{x}; w), f(\mathbf{x}; w^*)\big)
\Bigg).
\end{align}Here, $\widetilde{w}$ is the optimal model parameters of approximate unlearning.

\subsection{Streaming Unlearning Setup}
The discussions so far have assumed a \emph{batch unlearning} setting, where the forgetting data $F$ is specified once and the dataset $D$ remains accessible during unlearning. However, in many practical scenarios, deletion requests arrive sequentially, and storing or accessing the original training dataset is impractical or prohibited. This motivates the development of \emph{streaming unlearning}, which aims to efficiently update the model as new deletion requests are continuously received.

In streaming unlearning, a training dataset $D$ is used to train the initial model $f(\cdot; w_0)$. 
After learning $w_0$, a sequence of unlearning requests $\{F_t\}_{t=1}^T$ is received, where $F_t$ contains the data to be forgotten at step $t$. 
Let the remaining data at step $t$ be $D_t = D_{t-1} \setminus F_t$. 
The exact unlearning objective at step $t$ is then:
\begin{align}\label{eq:w_star}
w_t^* = \arg\min_w \mathcal{L}(D_t, w).
\end{align}
Without loss of generality, we assume $F_{t} \subseteq D_{t-1}$ at each step. If not, any data points in $F_{t}$ that do not belong to $D_{t-1}$ are ignored.

Although exact unlearning requires retraining from scratch at each step to obtain $w_t^*$, 
this is computationally impractical, especially in streaming settings where unlearning requests arrive continuously.
Therefore, similar to the batch unlearning case, we seek to approximate the behavior of $w_t^*$ by minimizing the discrepancy between the current model $f(\cdot; w)$ and $f(\cdot; w_t^*)$. At step $t$, this leads to the following objective:
\begin{align}\nonumber
%\widetilde{w}_t = 
\arg\min_{w} \Bigg(
& \sum_{(\mathbf{x}, y) \in {D}_t} \mathrm{diff}\big(f(\mathbf{x}; w), f(\mathbf{x}; w_t^*)\big) \\\label{eq:onlineuobjround}
&+ \lambda  \sum_{(\mathbf{x}, y) \in \bigcup_{i=1}^t F_i} \mathrm{diff}\big(f(\mathbf{x}; w), f(\mathbf{x}; w_t^*)\big)
\Bigg).
\end{align}
Here, the forgetting term accumulates over all previously deleted data points up to step $t$, i.e., $\bigcup_{i=1}^t F_i$.
Over all $T$ steps, the cumulative objective is to minimize:
\begin{align}\nonumber
\mathbf{Obj}_{\text{stream}} = \sum_{t=1}^{T} \Bigg(
& \sum_{(\mathbf{x}, y) \in {D}_t} \mathrm{diff}\big(f(\mathbf{x}; w), f(\mathbf{x}; w_t^*)\big) \\\label{eq:onlineuobj}
&+ \lambda  \sum_{(\mathbf{x}, y) \in \bigcup_{i=1}^t F_i}  \mathrm{diff}\big(f(\mathbf{x}; w), f(\mathbf{x}; w_t^*)\big)
\Bigg).
\end{align}
\paragraph{Relationship to Online Learning.}  
Streaming unlearning shares key characteristics with \emph{online learning}~\cite{onlinelds3,onlinelds2,onlinelds1}, notably the sequential arrival of data. The main distinction lies in the data dynamics: online learning progressively augments the training set, whereas streaming unlearning sequentially removes data points. Despite this fundamental difference, the concept of \emph{regret}—which quantifies the cumulative performance gap between sequentially updated models and the optimal models at each step—remains relevant. This relevance is reflected in the streaming unlearning objective $\mathbf{Obj}_{\text{stream}}$ (Eq.~\ref{eq:onlineuobj}), which mirrors dynamic regret by accumulating, over time, the discrepancies between approximate models $w$ and the per-step optimal models $w_t^*$.

\section{SAFE: Streaming Unlearning Methodology}\label{sec:method}

In this section, we develop a systematic methodology for streaming unlearning. We first propose a risk estimator to quantify the performance of approximate unlearning solutions (Sec.~\ref{riskestimator}), and explain how it can be incrementally estimated throughout the streaming process (Sec.~\ref{distributionshift}). We then give the theoretical analyses establishing its performance guarantees (Sec.~\ref{theoretical}).

\subsection{Streaming Unlearning Risk Estimator}\label{riskestimator}

To systematically evaluate the performance of streaming unlearning, we first formalize a per-round risk estimator. This risk function will serve as a more concrete objective for algorithm design and theoretical analysis in subsequent subsections. As defined in Eq.~\ref{eq:onlineuobjround}, the per-round objective minimizes discrepancy terms that quantify how well the current model preserves retention and enforces forgetting relative to the ideal unlearned model. We now specify their concrete forms. 

For the remaining data $D_t$, we define:
\begin{align}
&\sum_{(\mathbf{x}, y) \in {D}_t} \mathrm{diff}\big(f(\mathbf{x}; w), f(\mathbf{x}; w^*)\big) \quad\notag \\\label{eq:diffDt}
&=  \frac{1}{|D_t|}\sum_{(\mathbf{x}, y) \in D_t} \ell(f(\mathbf{x}; w), y) 
 -  \frac{1}{|D_t|}\sum_{(\mathbf{x}, y) \in D_t} \ell(f(\mathbf{x}; w_t^*), y), 
\end{align}
which quantifies the difference between the task-specific losses $\ell$ of the current model $w$ and the ideal unlearned model $w_t^*$ over the remaining data $D_t$.

For the forgetting data $\bigcup_{i=1}^t F_i$, we define:
\begin{align}
&\sum_{(\mathbf{x}, y) \in \bigcup_{i=1}^t F_i} \mathrm{diff}\big(f(\mathbf{x}; w), f(\mathbf{x}; w^*)\big)\quad\quad\quad\quad \notag \\
&\quad\quad = \frac{1}{\sum_{i=1}^t |F_i|}  \sum_{(\mathbf{x}, y) \in {\bigcup_{i=1}^t F_i}} d_{\mathrm{KL}}(f(\mathbf{x}; w), f(\mathbf{x}; w_t^*)), \label{eq:diffFt}
\end{align}
where $d_{\mathrm{KL}}$ denotes the Kullback-Leibler divergence between the predictive distributions under parameters $w$ and $w_t^*$.

The discrepancy terms in Eq.~\ref{eq:diffDt} and Eq.~\ref{eq:diffFt} are defined differently to reflect the distinct roles of $D_t$ and $F_t$. For $D_t$, the objective is to preserve task performance, naturally measured by the prediction loss difference~\cite{salun,camu}. For $F_t$, since the goal is to enforce forgetting and label-based losses may no longer be meaningful or accessible, we assess the divergence between the model’s predictive distributions~\cite{ts,camu}. Based on these definitions, we express the per-round objective as
\begin{align}
&\quad\arg\min_{w}
 \frac{1}{|D_t|} \sum_{(\mathbf{x}, y) \in D_t} \left(\ell(f(\mathbf{x}; w), y) - \ell(f(\mathbf{x}; w_t^*), y)\right)\notag \\
&\quad\quad +\frac{\lambda}{\sum_{i=1}^t |F_i|}  \sum_{(\mathbf{x}, y) \in {\bigcup_{i=1}^t F_i}}  d_{\mathrm{KL}}(f(\mathbf{x}; w), f(\mathbf{x}; w_t^*))  \notag\\
=&\quad\arg\min_{w_t} \frac{1}{|D_t|} \sum_{(\mathbf{x}, y) \in D_t} \ell(f(\mathbf{x}; w), y) \notag\\
&\quad\quad +\frac{\lambda}{\sum_{i=1}^t |F_i|}  \sum_{(\mathbf{x}, y) \in {\bigcup_{i=1}^t F_i}}  d_{\mathrm{KL}}(f(\mathbf{x}; w), f(\mathbf{x}; w_t^*)) .
\label{eq:argmin_equal}
\end{align}
The equality holds because $\ell(f(\mathbf{x}; w_t^*)$ does not depend on $w$ and thus does not affect the minimization. Thus we define the risk function as
\begin{align}
R_t(w) 
&= \frac{1}{|D_t|} \sum_{(\mathbf{x}, y) \in D_t} \ell(f(\mathbf{x}; w), y) \notag \\
& +  \frac{\lambda}{\sum_{i=1}^t |F_i|} \sum_{i=1}^t\sum_{(\mathbf{x}, y) \in F_i} d_{\mathrm{KL}}(f(\mathbf{x}; w), f(\mathbf{x}; w_t^*)). \label{eq:Rt}
\end{align}

To better understand the temporal evolution of the risk, we express $R_t(w)$ as a cumulative function of the initial risk $R_0(w)$ and the accumulated losses over $F_t$. Since the remaining data evolves as $D_t = D_{t-1} \setminus F_t$, the cumulative loss over $D_t$ can be written as:
\begin{align}
\frac{1}{|D_t|} \sum_{(\mathbf{x}, y) \in D_t} \ell(f(\mathbf{x}; w), y) 
&= \frac{1}{|D_t|} \sum_{(\mathbf{x}, y) \in D_0} \ell(f(\mathbf{x}; w), y) \notag \\
- & \frac{1}{|D_t|} \sum_{i=1}^t \sum_{(\mathbf{x}, y) \in F_i} \ell(f(\mathbf{x}; w), y). \label{eq:Dtdecompose}
\end{align}
Also note that $F_0 = \emptyset$ and 
\begin{align}
R_0(w) 
&= \frac{1}{|D_0|} \sum_{(\mathbf{x}, y) \in D_0} \ell(f(\mathbf{x}; w), y) \label{eq:R0}. 
\end{align}
Replacing Eqs.~\ref{eq:Dtdecompose} and ~\ref{eq:R0} into Eq.~\ref{eq:Rt} leads to:
\begin{align}\nonumber
R_t(w) = &\underbrace{\frac{|D_{0}|}{|D_{t}|}R_{0}(w) - \frac{1}{|D_{t}|}\sum_{i=1}^{t}\sum_{(\mathbf{x},y)\in F_{i}}\ell(f(\mathbf{x}; w),y)}_{(a)  \text{ Retention  Term}} \\
&\underbrace{+ \frac{\lambda}{\sum_{i=1}^t |F_i|} \sum_{i=1}^{t}\sum_{(\mathbf{x},y)\in F_{i}}d_{\mathrm{KL}}(f(\mathbf{x};w), f(\mathbf{x};w^*_{t}))}_{(b) \text{ Forgetting  Term}}.\label{eq:trisk}
\end{align}

The risk estimator in Eq.~\ref{eq:Rt} provides a principled objective that balances preserving task performance on remaining data and promoting forgetting on deleted data. However, the main challenge lies not in formulating the risk itself, but in optimizing it efficiently without storing the original dataset $D_0$, as retaining and repeatedly computing over such a large dataset would be impractical in streaming scenarios. 

In practice, streaming unlearning tasks often involve small deletion requests (i.e., $|F_i|$ is relatively small), while imposing strict constraints on real-time responsiveness and memory efficiency. In the following sections, we develop an optimization approach that incrementally updates model parameters and risk estimates without needing access to the original training set.

\subsection{Incremental Risk Optimization}\label{distributionshift}

\subsubsection{Optimizing the Retention Term}\label{sec:opt_retent}
For ease of presentation, we rewrite the retention term in Eq.~\ref{eq:trisk} here
\begin{align}
R_t^{\text{ret}}(w) = \frac{|D_{0}|}{|D_{t}|}R_{0}(w) - \frac{1}{|D_{t}|}\sum_{i=1}^{t}\sum_{(\mathbf{x},y)\in F_{i}}\ell(f(\mathbf{x}; w),y). \label{eq:Rret}
\end{align}
Optimizing the retention term in Eq.~\ref{eq:Rret} efficiently in a streaming unlearning setting poses significant challenges. At each deletion step, calculating the retention term’s gradient and updating the model require efficient strategies that avoid recomputing over large datasets or storing excessive historical data. Regarding the second component in Eq.~\ref{eq:Rret}, calculating it requires using the deletion sets $\{F_1, \dots, F_t\}$. Since each $F_i$ and their cumulative size remain relatively small, storage costs could be manageable. Nonetheless, to ensure computational efficiency, it is better to incorporate their contribution through a recursive update mechanism. For the first component in Eq.~\ref{eq:Rret}, evaluate it and its gradient request a calculation over the entire initial dataset $D_0$. In real-world learning tasks, $D_0$ is typically large. Storing or repeatedly accessing it is not feasible. Beyond the storage burden, recomputing gradients over such a large dataset at each round would cause significant latency, especially given the high frequency of deletion requests in streaming unlearning.

To overcome these challenges, we adopt an incremental gradient update strategy. Since $w_0$ was already trained to perform well on the remaining data, each optimization step begins from $w_0$ to maintain performance while improving efficiency. At each time step, we perform a single gradient descent update starting from $w_0$, following a one-step update strategy commonly adopted in online learning~\cite{nonstaionary,ons}. While limiting updates to a single step may introduce some performance trade-offs, Sec.~\ref{theoretical} later shows that, with a properly chosen learning rate, this approach achieves satisfactory results with theoretical guarantees. This update strategy significantly reduces both storage requirements and computational complexity, supporting real-time updates in streaming unlearning without compromising retention performance.

Specifically, the incremental gradient update proceeds as follows. Before the streaming unlearning starts, We compute $\nabla_{w_0} R_0(w_0)$ once using the training dataset $D_0$. $D_0$ is not used anymore in later updates. As new deletion requests arrive, we compute only the gradient of the newly deleted data $F_t$ with respect to $w_0$. The retention term’s gradient is updated recursively:
\begin{align}
\nabla_{w_0} R_t^{\text{ret}}(w_0) 
&= \frac{|D_{t-1}|}{|D_t|} \nabla_{w_0} R_{t-1}^{\text{ret}}(w_0) \notag \\
&\quad - \frac{1}{|D_t|} \sum_{(\mathbf{x}, y) \in F_t} \nabla_{w_0} \ell(f(\mathbf{x}; w_0), y). \label{eq:ret_recursive}
\end{align}
In this recursive update, $|D_{t-1}|$, which can be updated incrementally as $|D_t| = |D_{t-1}| - |F_t|$ at each step. For each new deletion request, it only needs to compute the gradient of the newly removed data $F_t$ with respect to $w_0$. Since $|F_t|$ is typically small, the additional computational cost per round remains low. Thus, the proposed recursive update mechanism ensures constant storage and per-step computation that scales with $|F_t|$.

\subsubsection{Optimizing the Forgetting Term}

\paragraph{Forgetting term \& challenge} For clarity, we restate the forgetting term from Eq.~\ref{eq:trisk} as:
\begin{align}
R_t^{\text{fg}}(w) = \frac{\lambda}{\sum_{i=1}^t |F_i|} \sum_{i=1}^t \sum_{(\mathbf{x}, y) \in F_i} d_{\mathrm{KL}}(f(\mathbf{x}; w), f(\mathbf{x}; w_t^*)). \label{eq:Rforg}
\end{align}

Unlike the retention term in Eq.~\ref{eq:Rret}, this component does not depend on the original dataset $D_0$. Instead, the main challenge arises from the second argument of the KL divergence, $f(\mathbf{x}; w_t^*)$, which represents the predictions of the ideal unlearned model retrained on the remaining data $D_t$ (as defined in Eq.~\ref{eq:w_star}). However, $w_t^*$ is the very objective that the unlearning process seeks to approximate and inherently unavailable at this stage.

To address this, we note that, when ideally trained, $f(\mathbf{x}; w_t^*) \approx p_t(y|\mathbf{x})$, where $p_t(y|\mathbf{x})$ denotes the conditional probability estimated from the remaining data $D_t$. Similarly, $f(\mathbf{x}; w_0) \approx p_0(y|\mathbf{x})$, where $p_0(y|\mathbf{x})$ is estimated from the original dataset $D_0$. We therefore propose to construct a relationship between $p_0(y|\mathbf{x})$ and $p_t(y|\mathbf{x})$, enabling us to approximate $f(\mathbf{x}; w_t^*)$ on $F_t$ using $f(\mathbf{x}; w_0)$, despite the inaccessibility of $w_t^*$. To further ensure the method’s practicality in streaming scenarios, the approximation should also support recursive updates with minimal storage and computation as new deletion requests arrive.

\paragraph{Posterior shift approximation} We then employ Bayes' theorem to relate $p_0(y|\mathbf{x})$ and $p_t(y|\mathbf{x})$:
\begin{align}\label{eq:ds}
    p_t(y|\mathbf{x}) 
    &= \frac{p_0(\mathbf{x})}{p_t(\mathbf{x})}\frac{p_t(y)}{p_0(y)}\frac{p_t(\mathbf{x}|y)}{p_0(\mathbf{x}|y)}p_0(y|\mathbf{x}). 
\end{align}
Motivated by density ratio modeling approaches commonly used in domain adaptation and distribution shift research~\cite{quinonero2008covariate,sugiyama2007direct}, we approximate the posterior shift using only the label and conditional distribution ratios, while marginal feature shifts are either assumed negligible or absorbed into the proportional constant. That is,
\begin{align}\label{eq:ratio}
p_t(y|\mathbf{x}) 
&\propto \frac{p_t(y)}{p_0(y)} \cdot \frac{p_t(\mathbf{x}|y)}{p_0(\mathbf{x}|y)}\cdot p_0(y|\mathbf{x}). 
\end{align}
If the density ratios $\frac{p_t(y)}{p_0(y)}$ and $\frac{p_t(\mathbf{x}|y)}{p_0(\mathbf{x}|y)}$ can be estimated, the initial model $f(\mathbf{x}; w_0)$ can be adjusted (up to normalization) to approximate the target model $f(\mathbf{x}; w_t^*)$, i.e., 
\begin{align}\label{eq:model_adjust}
    f(\mathbf{x}; w_t^*) \approx \frac{p_t(y)}{p_0(y)} \cdot \frac{p_t(\mathbf{x}|y)}{p_0(\mathbf{x}|y)} \cdot f(\mathbf{x}; w_0) \cdot \text{(normalization)}.
\end{align}

The first term, $\frac{p_t(y)}{p_0(y)}$, reflects the change in the marginal label distribution between $D_0$ and $D_t$. Specifically, it corresponds to the ratio of the proportions of data belonging to class $y$ in the two datasets:
\begin{align}\label{eq:label_ratio}
    \frac{p_t(y)}{p_0(y)} \approx \frac{n_t(y) / |D_t|}{n_0(y) / |D_0|} = \frac{n_t(y)}{n_0(y)} \cdot \frac{|D_0|}{|D_t|},
\end{align}
where $n_t(y)$ and $n_0(y)$ denote the number of samples with label $y$ in $D_t$ and $D_0$, respectively, and $|D_t|$ and $|D_0|$ are the total sample sizes of the two datasets.

\paragraph{Incremental estimation of class-conditional distributions} The second term, $\frac{p_t(\mathbf{x}|y)}{p_0(\mathbf{x}|y)}$, is generally more challenging to estimate directly. To make this estimation tractable and adaptable to incremental changes, we propose approximating $p_t(\mathbf{x}|y)$ and $p_0(\mathbf{x}|y)$ using a parametric distribution whose parameters can be incrementally updated as data evolves.

Previous research has shown that Gaussian distributions possess this desirable property. In particular, the incremental estimation of Gaussian distributions~\cite{DBLP:conf/colt/DasguptaH07,finch2009incremental} allows the parameters of the updated distribution to be derived from the previous distribution and newly arriving data. Although our setting involves data deletion rather than addition, a similar incremental update mechanism can apply if we can model $p_t(\mathbf{x}|y)$ and $p_0(\mathbf{x}|y)$ as Gaussian distributions.

However, as raw data distributions are rarely Gaussian, we follow the common practice in variational autoencoders (VAE)~\cite{vae} and Bayesian Neural Networks~\cite{neal2012bayesian} of projecting data into a latent space that approximates a Gaussian distribution. 

Through experiments, we found that the projected data $\mathbf{z}$ approximately follows a Gaussian distribution despite data removal\footnote{The Gaussianity of the standardized projected features $\mathbf{z}$ is validated using Mardia's multivariate normality test~\cite{mardia1970measures}. As shown in Appendix~B.6 Table~6, the resulting $p$-values for both skewness and kurtosis are consistently high across all evaluated scenarios, indicating that the null hypothesis of multivariate normality cannot be rejected.}. This observation aligns with prior findings that low-dimensional random linear combinations of data tend to approach maximal entropy (Gaussian) distributions~\cite{friedman1974projection}. Specifically, we first generate a random projection matrix $\mathbf{V}$, where each entry is independently sampled from a standard Gaussian distribution. At each step $t$, we project $\mathbf{x}$ using $\mathbf{V}$ and then standardize the projected data across the entire dataset $D_t$. This yields the transformed features $\mathbf{z}$, which is
\begin{align}\label{eq:standardize}
    \mathbf{z} = \boldsymbol{\Sigma}^{-1/2} \left( \mathbf{V}^\top \mathbf{x} - \boldsymbol{\mu} \right)
\end{align}
where $\boldsymbol{\mu}$ and $\boldsymbol{\Sigma}$ represent the empirical mean and covariance matrix of the projected features $\mathbf{V}^\top \mathbf{x}$ for data points in $D_t$ with label $y$ at step $t$. For clarity, we denote these as $\boldsymbol{\mu}_t^{(y)}$ and $\boldsymbol{\Sigma}_t^{(y)}$ in the following discussions to reflect their dependence on both the current dataset $D_t$ and the class $y$.

Since data deletion requests may lead to changes in the class-conditional distribution over time, it is essential to update $\boldsymbol{\mu}_t^{(y)}$ and $\boldsymbol{\Sigma}_t^{(y)}$ sequentially rather than recomputing them from scratch at each step. Let $n_t(y)$ denote the number of data points in $D_t$ with label $y$. So the number of data points with label $y$ that were removed between step $t-1$ and step $t$ is $n_t(y)-n_{t-1}(y)$. We also define $\boldsymbol{\widetilde{\mu}}_t^{(y)}$ and $\boldsymbol{\widetilde{\Sigma}}_t^{(y)}$
 as the mean vector and covariance matrix of the data points in $F_t$ with label $y$. Then the updated mean can be recurrently computed as:
\begin{align}\label{eq:mu_update}
    \boldsymbol{\mu}_t^{(y)} = \frac{n_{t-1}(y) \cdot \boldsymbol{\mu}_{t-1}^{(y)} - \left( n_{t-1}(y) - n_t(y) \right) \cdot \boldsymbol{\widetilde{\mu}}_{t}^{(y)}}{n_t(y)}, 
\end{align}
and the covariance matrix can be recurrently updated as:
\begin{align}\label{eq:sigma_update}
    \boldsymbol{\Sigma}_t^{(y)} = \frac{n_t(y) - 1}{n_{t-1}(y) - 1} \cdot \boldsymbol{\Sigma}_{t-1}^{(y)} - \boldsymbol{\widetilde{\Sigma}}_{t}^{(y)} + c\left(\boldsymbol{\mu}_t^{(y)}, \boldsymbol{\mu}_{t-1}^{(y)}, \boldsymbol{\widetilde{\mu}}_{t}^{(y)} \right),
\end{align}where the correction term $c(\cdot)$ accounts for the change in the mean and is defined as:
\[
\begin{aligned}
    c\big( \boldsymbol{\mu}_t^{(y)}, \boldsymbol{\mu}_{t-1}^{(y)}, \boldsymbol{\widetilde{\mu}}_{t}^{(y)} \big) 
    &= \frac{n_{t-1}(y) \cdot \big( n_{t-1}(y) - n_t(y) \big)}{n_t(y) \cdot \big( n_{t-1}(y) - 1 \big)} \\
    &\quad \cdot \big( \boldsymbol{\mu}_{t-1}^{(y)} - \boldsymbol{\widetilde{\mu}}_{t}^{(y)} \big) \big( \boldsymbol{\mu}_{t-1}^{(y)} - \boldsymbol{\widetilde{\mu}}_{t}^{(y)} \big)^\top.
\end{aligned}
\]
Eqs~\ref{eq:mu_update} and \ref{eq:sigma_update} follow directly from the definitions of the mean and covariance with data removal, and we omit the detailed derivation.

Since $p_t(\mathbf{x}|y)$ and $p_0(\mathbf{x}|y)$ are both Gaussian distributions, their density ratio can be computed in closed form. Combining this with Eq.~\ref{eq:model_adjust} and Eq.~\ref{eq:label_ratio}, we have
\begin{align}
f(\mathbf{x}; w_t^*) \approx q_t^{(y)}(\mathbf{x}) f(\mathbf{x}; w_0)
\end{align}
where
\[
q_t^{(y)}(\mathbf{x}) =  
\frac{n_t(y)}{n_0(y)} \cdot \frac{|D_0|}{|D_t|} 
\cdot \mathcal{N} \left( \mathbf{z} \mid \boldsymbol{\mu}_t^{(y)}, \boldsymbol{\Sigma}_t^{(y)} \right) 
/ \mathcal{N} \left( \mathbf{z} \mid \boldsymbol{\mu}_0^{(y)}, \boldsymbol{\Sigma}_0^{(y)} \right) 
\]
and we omit the normalization factor for ease of presentation. 

\paragraph{Optimization} By replacing $f(\mathbf{x}; w_t^*)$ in the forgetting regularizer (Eq.~\ref{eq:Rforg}), we obtain an approximate regularizer:
\begin{align}\label{eq:apprx_fg_risk}
\widetilde{R}_t^{\text{fg}}(w) 
= \frac{\lambda}{\sum_{i=1}^t |F_i|} \sum_{i=1}^t \sum_{(\mathbf{x}, y) \in F_i} d_{\mathrm{KL}}\big(f(\mathbf{x}; w), \, q_t^{(y)}(\mathbf{x}) \cdot f(\mathbf{x}; w_0)\big).
\end{align}
We minimize $\widetilde{R}_t^{\text{fg}}(w)$ by performing a single gradient descent step starting from $w_0$, following the same strategy in Sec~\ref{sec:opt_retent}. This approach is motivated by the fact that $w_0$ has already been trained to fit the original data distribution, and the adjusted target distribution $q_t^{(y)}(\mathbf{x}) \cdot f(\mathbf{x}; w_0)$ only reweights the original predictions without introducing significant changes to the overall prediction structure. Therefore, initializing from $w_0$ and applying a small adjustment suffices to align $f(\mathbf{x}; w)$ with the target distribution. The gradient is:
\begin{align}\label{eq:weight_update}
 \nabla_w \widetilde{R}_t^{\text{fg}}(w) \big|_{w = w_0}.
\end{align} 

The storage and computational cost for this update process primarily involves evaluating $\widetilde{R}_t^{\text{fg}}(w)$ and its gradient at $w_0$, which requires accessing the forgetting set $F_i$ and their associated statistics (e.g., $q_t^{(y)}(\mathbf{x})$). However, as the forgetting sets $F_i$ are typically much smaller than the full dataset $D_0$, the overall storage burden remains manageable. Furthermore, the Gaussian parameters $\boldsymbol{\mu}_t^{(y)}$, $\boldsymbol{\Sigma}_t^{(y)}$ are maintained and updated incrementally, which significantly reduces the need for recomputation. In future work, we plan to explore fully recursive strategies to further improve efficiency.

\subsubsection{Overall Algorithm} 
Combining the optimization strategies for both the retention term $R_t^{\text{ret}}(w)$ and the forgetting term $\widetilde{R}_t^{\text{fg}}(w)$, we adopt an incremental update approach. Upon each forgetting request, a single gradient descent update is performed starting from $w_0$. The update gradient aggregates the retention gradient in Eq.~\ref{eq:ret_recursive} and the forgetting gradient in Eq.~\ref{eq:weight_update}. Additionally, a Gaussian perturbation $b_t \sim \mathcal{N}(0, \phi_t)$ (where $\phi_t$ is a parameter) is introduced to add stochasticity and mitigate potential overfitting to individual forgetting request. Overall, the update rule is:
\begin{align}\label{eq:gd_overall}
w_t = w_0 - \gamma \cdot \frac{g_t}{\|g_t\|_2} - b_t,
\end{align}
where
\[
g_t = \nabla_{w_0} R_t^{\text{ret}}(w_0) + \nabla_w \widetilde{R}_t^{\text{fg}}(w) \big|_{w = w_0}
\]
and $\|\cdot\|_2$ is the L2 norm. The complete procedure is summarized in Algorithm~\ref{alg:Framework}.
\begin{algorithm}[t!]
\caption{SAFE: Streaming Forgetting Method}
\label{alg:Framework}
\begin{algorithmic}[1]
\Require $|D_0|$, $\{F_t\}_{t=1}^{T}$, $\mathbf{V}$, $\boldsymbol{\mu}$, $\boldsymbol{\Sigma}$, $w_0$, $\nabla_{w_0} R_0^{\text{ret}}(w_0)$, $\gamma$, $\phi_t$.
\Ensure $\{w_{t}\}_{t=1}^{T}$.
\State Initialize $w_0$ as the original model parameters.
\State Initialize $\boldsymbol{\mu}_0^{(y)}$, $\boldsymbol{\Sigma}_0^{(y)}$ from $\boldsymbol{\mu}$ and $\boldsymbol{\Sigma}$ for each class $y$.
\For{$t = 1,\ldots, T$}
    \State Update retention gradient $\nabla_{w_0} R_t^{\text{ret}}(w_0)$ using Eq.~\ref{eq:ret_recursive}.
    \State Update $\boldsymbol{\mu}_t^{(y)}$, $\boldsymbol{\Sigma}_t^{(y)}$ using Eqs.~\ref{eq:mu_update} and \ref{eq:sigma_update}.
    \State Estimate $q_t^{(y)}(\mathbf{x})$ for $(\mathbf{x}, y) \in F_t$ and compute $\widetilde{R}_t^{\text{fg}}(w)$.
    \State Calculate total gradient:
    \[
    g_t = \nabla_{w_0} R_t^{\text{ret}}(w_0) + \nabla_w \widetilde{R}_t^{\text{fg}}(w) \big|_{w = w_0}.
    \]
    \State Sample Gaussian perturbation $b_t \sim \mathcal{N}(0, \phi_t)$.
    \State Update model: $w_t = w_0 - \gamma \cdot \frac{g_t}{\|g_t\|_2} - b_t$.
    
\EndFor
\end{algorithmic}
\end{algorithm}

\subsection{Theoretical Results}\label{theoretical}
For the streaming unlearning task, the per-step objective is the risk estimator $R_t(w)$ defined in Eq.~\ref{eq:trisk}. In practice, we approximate $f(\mathbf{x}; w_t^*)$ by $q_t^{(y)}(\mathbf{x}) f(\mathbf{x}; w_0)$, resulting in the following surrogate objective:
\begin{align}
\widetilde{R}_{t}(w) =  R_t^{\text{ret}}(w) + \widetilde{R}_t^{\text{fg}}(w), \label{eq:erisk2}
\end{align}
where the two terms are defined in Eqs.~\ref{eq:Rret} and~\ref{eq:apprx_fg_risk}, respectively.

Our first theoretical result shows that $\widetilde{R}_{t}(w)$ closely approximates $R_t(w)$:
\begin{theorem}\label{risktheorem}
If $p_0(y|\mathbf{x}) = f(\mathbf{x}; w_0)$, then $|\widetilde{R}_{t}(w) - R_{t}(w)| \leq C \cdot \sum_{i = 1}^{t} |\mathcal{F}_{i}| / |\mathcal{D}_{t}|^{3/2}$, where $C$ is the number of classes.
\end{theorem}
\noindent The detailed proof is provided in Appendix A.3. Theorem~\ref{risktheorem} indicates that the surrogate risk $\widetilde{R}_{t}(w)$ converges to the true risk $R_t(w)$ as the proportion of deleted data relative to the remaining data decreases.

After addressing the approximation of the true risk estimator, we further establish theoretical guarantees for the algorithm's additional approximations. Despite adopting several simplifications---including updating with only a single gradient descent step and modeling the data distribution using a Gaussian approximation---the algorithm's performance is still provably efficient under mild conditions. Before presenting the detailed performance guarantees, we formally define $(\epsilon, \delta)$-approximate unlearning, following~\cite{DBLP:conf/icml/GuoGHM20}:\vspace{-1mm}
\begin{definition}
Let $\mathsf{M}$ denote an unlearning algorithm that, given the original model $\mathsf{A}(D)$, the full dataset $D$, and the forgetting data $F$, produces an unlearned model. For any measurable subset of hypotheses $\mathcal{T} \subseteq \mathcal{H}$:
\begin{align*}
    e^{-\epsilon} \cdot \Pr\left[ \mathsf{A}(D \setminus F) \in \mathcal{T} \right] - \delta
    &\leq \Pr\left[ \mathsf{M}(\mathsf{A}(D), D, F) \in \mathcal{T} \right] \\
    &\leq e^{\epsilon} \cdot \Pr\left[ \mathsf{A}(D \setminus F) \in \mathcal{T} \right] + \delta.
\end{align*}
Here, $\Pr[\cdot]$ denotes the probability that a model belongs to $\mathcal{T}$, which represents a set of desirable unlearned models. The parameter $\epsilon \geq 0$ controls the multiplicative similarity between the distributions (smaller $\epsilon$ indicates greater similarity), while $\delta \geq 0$ allows for a small additive slack.
\end{definition}\vspace{-1mm}We then establish the following performance guarantee:\vspace{-1mm}
\begin{theorem}\label{theorem:th2}
Suppose the risk $R_{t}(w)$ has gradients upper bounded by $U$, i.e., $\|\nabla R_{t}(w)\|_2 \leq U$, and the model parameters satisfy $\|w\|_2 \leq W$. Set the learning rate $\gamma = \frac{\sqrt{W}}{K\sqrt{T}}$ and the perturbation variance $\phi_t = \frac{W\sqrt{2\ln{(1.25/\delta)}}}{\epsilon}$, where $K$ is a constant. Let $w_t$ denote the output of Algorithm~\ref{alg:Framework} at step $t$. Then, the following hold:

\noindent(i) The expected error at step $t$ compared to the optimal model $w_t^*$ is upper bounded:
\begin{align*}
 \mathbb{E}\left[R_{t}(w_t) - R_{t}(w_t^*)\right] &\leq O(\sqrt{T}).
\end{align*}
\noindent(ii) The accumulated unlearning regret across $T$ steps is upper bounded:
\begin{align*}
 \mathbb{E}\left[\sum_{t=1}^{T} \left( R_{t}(w_t) - R_{t}(w_t^*) \right) \right] &\leq O(\sqrt{T} + V_T),
\end{align*}where $V_T = \sum_{t=1}^{T} \| w_t^* - w_{t-1}^* \|_2$.

\noindent(iii) The unlearning at each step satisfies $(\epsilon, \delta)$-\textit{approximate unlearning}.
\end{theorem}\vspace{-1mm}The cumulative error between the retrained model and the unlearned model is bounded by $O(\sqrt{T} + V_T)$. To the best of our knowledge, this is the first proven dynamic regret bound for the streaming unlearning problem. Previous online learning methods achieve dynamic regret bounds such as $O(\sqrt{V_T T})$~\cite{nonstaionary} and $O(T^{2/3})$~\cite{nonconvex1} without assuming convexity. In contrast, our bound of $O(\sqrt{T} + V_T)$ provides a tighter guarantee tailored to the streaming unlearning setting. Compared to batch machine unlearning baselines, which achieve bounds of $O(T^2)$~\cite{RememberWhatYouWanttoForget} and $O(T)$~\cite{DBLP:conf/icml/ChourasiaS23} in streaming scenarios, the SAFE algorithm's bound of $O(\sqrt{T} + V_T)$ is significantly tighter. 

\section{Experiments}

\subsection{Experiment Settings}
\textbf{Datasets and Models}: To validate the effectiveness of the proposed method, we conduct experiments on four datasets: MNIST~\cite{mnist}, Fashion~\cite{fashion}, CIFAR10~\cite{cifar}, and TinyImagenet~\cite{le2015tiny}. We use the official training sets of MNIST, Fashion, and CIFAR10 from PyTorch and TinyImagenet downloaded from Kaggle. In TinyImagenet, we randomly selected $80$K images to construct the training data. For the MNIST and Fashion datasets, we use a two-layer \textbf{CNN}~\cite{cnn}, while for the CIFAR10 and TinyImagenet datasets, we adopt a \textbf{ResNet-18} backbone~\cite{resnet} without the pre-trained weights. 

\noindent\textbf{Other Settings}: We conduct experiments on \emph{unlearning on the randomly selected subsets}, which includes the unlearning for $20$ rounds and randomly select $400$ data points for each round. Regarding hyperparameters, we tune $K$ in learning rate $\gamma$ to achieve better unlearning results, as detailed in Appendix B.5. We set $\lambda$ as $1000$ for MNIST and Fashion, then we set it as $120000$ for the CIFAR10 and TinyImagenet. All experiments are under $10$ different random seeds. 

\noindent\textbf{Baselines}: We compare SAFE\footnote{https://github.com/ShaofeiShen768/SAFE} with Retrain, which represents the standard results of retrained models, and two unlearning methods designed to handle streaming unlearning requests \textbf{without requiring remaining data}: Lcode~\cite{mehta2022deep} and Desc~\cite{dtd}. Additionally, we evaluate two other methods that \textbf{require access to the remaining data}: Unroll~\cite{unroll} and CaMU~\cite{camu}.

\noindent\textbf{Evaluations}: For effectiveness evaluations, we assess the unlearning algorithm using five metrics: $\textbf{RA}$ (Remaining Accuracy), $\textbf{FA}$ (Forgetting Accuracy), and $\textbf{TA}$ (Test Accuracy), which denote the prediction accuracy of the post-unlearning model on the remaining data, forgetting data, and test data, respectively. We also evaluate the attack accuracy of $\textbf{MIA}$ (Membership Inference Attack)~\cite{miainu,mia} and follow the same MIA evaluation protocol as~\cite{salun,DBLP:conf/nips/JiaLRYLLSL23}. \textbf{Closer values to the retrained model} indicates better unlearning performance for these metrics. Additionally, we compare $\textbf{R}$ (average rank) of the previous four metrics for each method. The lower rank stands for the method ranks higher among all five methods.

\begin{figure*}[htbp]

       \vspace{-2mm}
	\centering
        \subfigure[Time in MNIST]{
		\begin{minipage}[t]{0.18\linewidth}
			\centering
			\includegraphics[width=\linewidth]{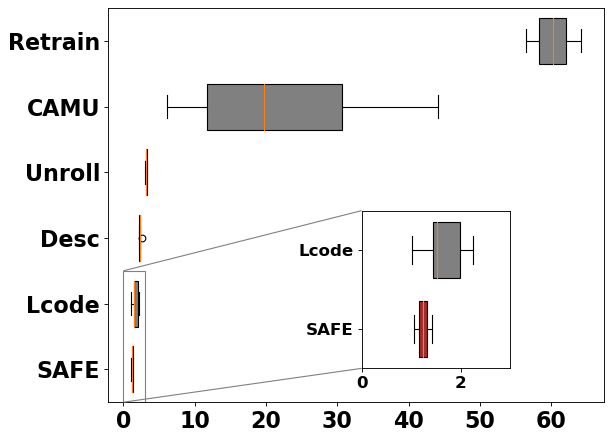}
		\end{minipage}
	}
        \hspace{2mm}
        \subfigure[Time in Fashion]{
		\begin{minipage}[t]{0.18\linewidth}
			\centering
			\includegraphics[width=\linewidth]{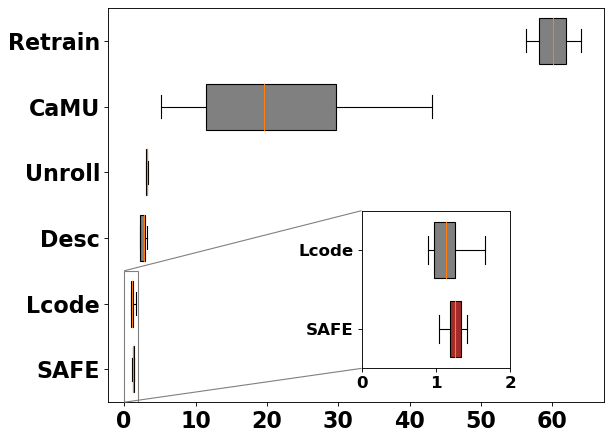}
		\end{minipage}
	}
        \hspace{2mm}
        \subfigure[Time in CIFAR10]{
		\begin{minipage}[t]{0.18\linewidth}
			\centering
			\includegraphics[width=\linewidth]{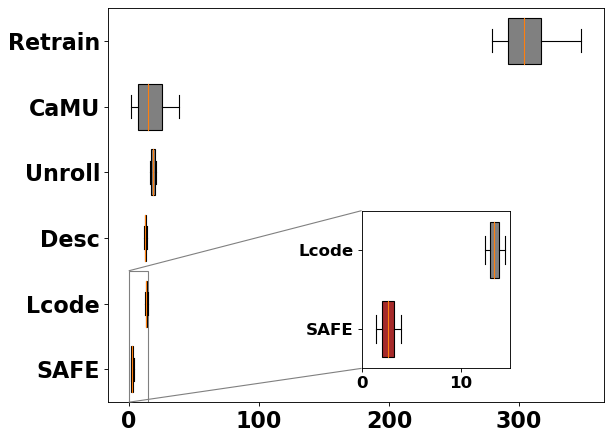}
		\end{minipage}
	}
        \hspace{2mm}
        \subfigure[Time in TinyImagenet]{
		\begin{minipage}[t]{0.18\linewidth}
			\centering
			\includegraphics[width=\linewidth]{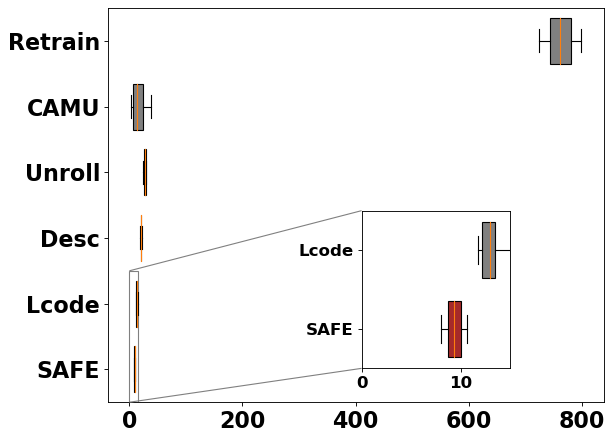}
		\end{minipage}
	}
        \vspace{-1mm}
	\caption{Comparisons of time costs for different unlearning algorithms (seconds) in streaming random subset unlearning. SAFE achieves the lowest time costs on MNIST, CIFAR10, and TinyImagenet, with at least $30$\% decrease compared with the method of the second-lowest cost.}
	\label{time}
    \vspace{1mm}
\end{figure*}

\subsection{Results Analysis}

\noindent\textbf{Accuracy Analysis:} We first evaluate the SAFE algorithm on the unlearning of sequential random subsets. Table~\ref{comparison1} presents the average performance across $20$ rounds of requests, each involving the removal of 400 randomly selected data points. The results indicate that SAFE consistently achieves performance closest to that of the retrained model across all datasets in most evaluations and achieves the highest average rank in all four datasets. Specifically, SAFE attains the best performance in $10$ evaluations and ranks the second best in $4$ evaluations out of the $16$ conducted. Notably, SAFE consistently delivers the best performance in terms of FA while maintaining superior TA. In terms of RA, all methods exhibit degradation following the unlearning of the forgetting data. However, SAFE demonstrates a notable ability to mitigate this degradation compared to alternative methods. These results collectively affirm that SAFE offers superior unlearning performance with minimal adverse effects on model prediction functionality, as evidenced by the RA, FA, and TA evaluations.

\noindent\textbf{MIA Analysis:} Table~\ref{comparison1} also shows that SAFE achieves the best MIA performance on the Fashion, CIFAR10 and TinyImagenet, underscoring its effectiveness on complex datasets. Compared to other methods, SAFE demonstrates greater stability in MIA. This stability is attributed to the fewer update steps required during unlearning, which results in less impact on model parameters and prediction outcomes.

\noindent\textbf{Efficiency Analysis:} We recorded the average time cost for $20$ rounds of unlearning, with the results in Figure~\ref{time}. SAFE achieved average costs of $1.13$ seconds on MNIST, while the second-fastest algorithm, Lcode, required $1.66$ seconds, much slower than SAFE. However, on the CIFAR10 dataset, the advantage of SAFE is even more pronounced. SAFE required only $2.55$ seconds per request, whereas the second-fastest algorithm, Descent, required $12.90$ seconds, which is nearly nine times slower than SAFE. Moreover, on the TinyImagenet dataset, SAFE achieved the highest efficiency, faster than all other methods. These time efficiency results highlight the leading advantage of SAFE in sequential requests, especially for complex datasets.

To further evaluate the SAFE algorithm, we present additional experiments under different request settings in Appendix B.4 and the results of streaming unlearning within the same class in Appendix B.7. These results further validate the effectiveness of SAFE.

\begin{table}[htbp]\footnotesize
\renewcommand{\arraystretch}{1.2}
        \vspace{2mm}
	\centering
	\caption{Comparison results in Random Subset Unlearning (avg\%$\pm$std\%). The \textbf{bold} record indicates the best, and the \underline{underlined} record indicates the second-best. The white backgrounds denote the unlearning methods that require remaining data in experiments, while the grey backgrounds denote the approaches without remaining data. SAFE can always achieve the best results in FA and achieve $3$ out of $4$ best results in TA. \textbf{The average ranks (R) of SAFE are all the highest among all five methods on different datasets}.}
        \vspace{3mm}
        {\fontsize{6.5}{7.5}\selectfont
	\begin{tabular}{p{0.78cm}<{\centering} | p{1.26cm}<{\centering} p{1.27cm}<{\centering} p{1.26cm}<{\centering} p{1.26cm}<{\centering} p{0.35cm}<{\centering} }
		\toprule%\hline    
            \multirow{2}{*}{\textbf{Method}} &\textbf{RA} &\textbf{FA} &\textbf{TA}  &\textbf{MIA}&\textbf{R}\\

            ~ &\multicolumn{5}{c}{\textbf{MNIST}} \\ 
            
            \midrule%\hline
            \textbf{Retrain}& 99.68$\pm$0.05& 98.89$\pm$0.09& 99.00$\pm$0.05& 79.25$\pm$1.14& -\\

            \textbf{Unroll}&  \underline{99.24$\pm$0.22}& \underline{98.91$\pm$0.15}& \textbf{98.61$\pm$0.19}& \underline{79.27$\pm$1.15}& \textbf{1.75}\\

            \textbf{CaMU}& 98.94$\pm$0.36& 98.72$\pm$0.79& 98.54$\pm$0.42& 79.05$\pm$1.13& 3.25\\
            \rowcolor{gray!20}
            
            \textbf{Lcode} & 96.26$\pm$1.95& 96.27$\pm$1.88& 95.60$\pm$1.94&  78.79$\pm$2.22& 5\\
            
            \rowcolor{gray!20}
            \textbf{Desc}& 98.78$\pm$0.53& 98.72$\pm$0.53& 98.27$\pm$0.46& \textbf{79.24$\pm$1.14}& 3\\
            
            \rowcolor{gray!40}
            \textbf{SAFE}& \textbf{99.25$\pm$0.13}& \textbf{98.89$\pm$0.12}& \underline{98.58$\pm$0.11}& 79.28$\pm$1.14& \textbf{1.75}\\\midrule
            
            ~ &\multicolumn{4}{c}{\textbf{Fashion}} \\ 
            
            \midrule%\hline
            \textbf{Retrain}& 96.44$\pm$0.15&90.76$\pm$0.33& 90.40$\pm$0.15& 79.57$\pm$0.51& -\\

            \textbf{Unroll}& 90.61$\pm$0.91& 89.08$\pm$0.94& 87.99$\pm$0.85& \textbf{79.14$\pm$0.65}& \underline{2.75}\\

            \textbf{CaMU}& \underline{91.32$\pm$0.36}& \underline{90.45$\pm$0.83}& \underline{89.00$\pm$0.40}& 78.07$\pm$0.53& 2.75\\

            \rowcolor{gray!20}
            \textbf{Lcode} & 86.21$\pm$4.27& 86.26$\pm$4.15& 82.28$\pm$3.89& 78.60$\pm$0.53& 4.5\\
            
            \rowcolor{gray!20}
            \textbf{Desc}& 89.52$\pm$1.38& 89.27$\pm$1.54& 87.40$\pm$1.10& 78.28$\pm$0.77&3.75\\
            
            \rowcolor{gray!40}
            \textbf{SAFE}& \textbf{92.03$\pm$0.19}& \textbf{90.88$\pm$0.34}& \textbf{89.22$\pm$0.19}& \underline{79.00$\pm$0.44}&\textbf{1.25}\\\midrule

            ~ &\multicolumn{4}{c}{\textbf{CIFAR10}} \\ 
            
            \midrule%\hline
            \textbf{Retrain}& 97.61$\pm$0.25& 91.78$\pm$0.49& 91.19$\pm$0.34& 64.38$\pm$1.34& -\\

            \textbf{Unroll}& 93.59$\pm$2.57& \underline{90.46$\pm$1.92}& \underline{86.97$\pm$2.12}& 74.96$\pm$2.95& \underline{2.75}\\

            \textbf{CaMU}& \textbf{95.71$\pm$1.09}& 93.13$\pm$3.32& 84.99$\pm$0.15& 74.54$\pm$2.95& 3.25\\

            \rowcolor{gray!20}
            \textbf{Lcode} & 23.95$\pm$3.71& 24.08$\pm$4.18& 23.33$\pm$3.54& 54.47$\pm$8.42&4.5\\
            
            \rowcolor{gray!20}
            \textbf{Desc}& 86.24$\pm$0.07& 84.94$\pm$0.30& 85.50$\pm$0.01& \textbf{71.95$\pm$7.67}&3\\
            
            \rowcolor{gray!40}
            \textbf{SAFE}& \underline{94.54$\pm$1.90}& \textbf{92.49$\pm$1.65}& \textbf{88.31$\pm$1.61}& \underline{72.63$\pm$2.96}&\textbf{1.5}\\\midrule

            ~ &\multicolumn{4}{c}{\textbf{TinyImagenet}} \\ 
            
            \midrule%\hline
            \textbf{Retrain}& 75.57$\pm$0.51& 42.71$\pm$0.57& 42.50$\pm$0.54& 29.69$\pm$3.31&-\\

            \textbf{Unroll}& \textbf{53.87$\pm$2.00}& \underline{48.62$\pm$1.99}& \underline{42.08$\pm$1.42}& \underline{24.20$\pm$0.83}&\underline{1.75}\\

            \textbf{CaMU}& 38.60$\pm$2.14& 36.17$\pm$3.78& 35.06$\pm$1.98& 39.07$\pm$3.87&3.25\\

            \rowcolor{gray!20}
            \textbf{Lcode} & 9.34$\pm$5.19& 9.17$\pm$5.21& 6.84$\pm$3.29& 54.67$\pm$5.81&5\\
            
            \rowcolor{gray!20}
            \textbf{Desc}& 27.50$\pm$11.18& 27.48$\pm$11.30& 24.63$\pm$9.65& 37.73$\pm$6.75&3.75\\
            
            \rowcolor{gray!40}
            \textbf{SAFE}& \underline{49.64$\pm$2.99}& \textbf{47.70$\pm$2.20}& \textbf{42.24$\pm$2.45}& \textbf{25.99$\pm$1.59}&\textbf{1.25}\\
                        
            \bottomrule

	\end{tabular}\label{comparison1}
    }
	\vspace{-8mm}
\end{table}

\subsection{Ablation Study}

Table~\ref{ablationpaper} presents the results of the ablation study for the proposed algorithm, where we sequentially remove the initial training data gradient (\textbf{TR}) and the distribution shift loss (\textbf{DS}). First, when the initial training data gradient is removed, there is a significant drop in accuracies across all datasets. These experimental results indicate that the distribution shift loss contributes significantly to the unlearning process. In addition, when the distribution shift loss is removed, the forgetting data accuracies on all four datasets are similar to the original forgetting data accuracies, implying that the information of streaming forgetting data has not been effectively removed from the model.
\begin{table}[h]\footnotesize
\renewcommand{\arraystretch}{1.2}
        \vspace{-1mm}
	\centering
	\caption{Ablation study results (avg\%$\pm$std\%). Removing TR reduces the performance on RA and TA, which means \textbf{reducing the prediction capability of the model} while removing the DS consistently leads to the highest FA across all datasets, demonstrating \textbf{its critical role in effective forgetting}.}
        \vspace{4mm}
        {\fontsize{6.5}{7.5}\selectfont
	\begin{tabular}{p{1.0cm}<{\centering} | p{1.4cm}<{\centering} p{1.4cm}<{\centering} p{1.4cm}<{\centering} p{1.4cm}<{\centering}  }
		\toprule%\hline    
            \multirow{2}{*}{\textbf{Method}} &\textbf{RA} &\textbf{FA} &\textbf{TA}  &\textbf{MIA}\\

            ~ &\multicolumn{4}{c}{\textbf{MNIST}} \\ 
            
            \midrule%\hline
            \textbf{Retrain}& 99.68$\pm$0.05& 98.89$\pm$0.09& 99.00$\pm$0.05& 79.25$\pm$1.14\\

            \textbf{w/o TR}& 98.68$\pm$0.08&98.12$\pm$0.25& 98.09$\pm$0.10& 79.31$\pm$1.15\\

            \textbf{w/o DS}& \textbf{99.35$\pm$0.08}& \underline{99.19$\pm$0.05}& \textbf{98.71$\pm$0.07}&  \underline{79.29$\pm$1.14}\\

            \rowcolor{gray!40}
            \textbf{SAFE}&  \underline{99.25$\pm$0.13}& \textbf{98.89$\pm$0.12}&  \underline{98.58$\pm$0.11}& \textbf{79.28$\pm$1.14}\\\midrule
            
            ~ &\multicolumn{4}{c}{\textbf{Fashion}} \\ 
            
            \midrule%\hline
            \textbf{Retrain}& 96.44$\pm$0.15&90.76$\pm$0.33& 90.40$\pm$0.15& 79.57$\pm$0.51\\

            \textbf{w/o TR}& 91.54$\pm$0.19&\underline{90.02$\pm$0.50}& 88.80$\pm$0.16& \underline{79.00$\pm$0.48}\\

            \textbf{w/o DS}& \textbf{92.40$\pm$0.39}&91.91$\pm$0.67& \textbf{89.48$\pm$0.33}& 78.47$\pm$0.48\\

            \rowcolor{gray!40}
            \textbf{SAFE}& \underline{92.03$\pm$0.19}& \textbf{90.88$\pm$0.34}& \underline{89.22$\pm$0.19}& \textbf{79.00$\pm$0.44}\\\midrule

            ~ &\multicolumn{4}{c}{\textbf{CIFAR10}} \\ 
            
            \midrule%\hline
            \textbf{Retrain}& 97.61$\pm$0.25& 91.78$\pm$0.49& 91.19$\pm$0.34& 64.38$\pm$1.34\\

            \textbf{w/o TR}& 91.22$\pm$2.27&89.20$\pm$1.76& 85.30$\pm$2.94& \underline{73.46$\pm$3.70}\\

            \textbf{w/o DS}& \underline{94.40$\pm$1.46}&\underline{94.06$\pm$1.22}& \textbf{88.43$\pm$1.18}& 73.88$\pm$3.59\\

            \rowcolor{gray!40}
            \textbf{SAFE}& \textbf{94.54$\pm$1.90} & \textbf{92.49$\pm$1.65} & \underline{88.31$\pm$1.61}& \textbf{72.63$\pm$2.96}\\\midrule
            
            ~ &\multicolumn{4}{c}{\textbf{TinyImagenet}} \\ 
            
            \midrule%\hline
            \textbf{Retrain}& 75.57$\pm$0.51& 42.71$\pm$0.57& 42.50$\pm$0.54& 29.69$\pm$3.31\\
            
            \textbf{w/o TR}& 48.39$\pm$2.99&\textbf{47.12$\pm$2.22}& \underline{41.68$\pm$2.44}& \underline{24.97$\pm$1.70}\\

            \textbf{w/o DS}& \textbf{49.93$\pm$4.05}&49.46$\pm$3.91& 43.56$\pm$3.24& 24.68$\pm$1.92\\

            \rowcolor{gray!40}
            \textbf{SAFE}& \underline{49.64$\pm$2.99}& \underline{47.70$\pm$2.20}& \textbf{42.24$\pm$2.45}& \textbf{25.99$\pm$1.59}\\
            
            \bottomrule

	\end{tabular}\label{ablationpaper}
    }
       \vspace{-6mm}
\end{table}

\section{Conclusion}

In this paper, we address the practical requirements of long sequential unlearning by introducing a streaming unlearning paradigm. This paradigm is designed to realize sequential unlearning requests with high forgetting accuracy and efficiency. We first conceptualize unlearning as the distribution shift problem and estimate the distribution of low-dimensional vectors of the training data of each class. Then, we propose a novel SAFE algorithm alongside a first-order optimization that can reach a low regret bound. We conducted extensive experiments and the results show that SAFE consistently achieves top or near-top performances across various evaluations, including more than double the time efficiency compared with the second-most efficient algorithm, demonstrating its clear advantages over other baseline methodologies.

%%% Use this environment to include acknowledgements (optional).
%%% This will be omitted in doubleblind mode.

% \begin{ack}
% By using the \texttt{ack} environment to insert your (optional) 
% acknowledgements, you can ensure that the text is suppressed whenever 
% you use the \texttt{doubleblind} option. In the final version, 
% acknowledgements may be included on the extra page intended for references.
% \end{ack}

%%%%%%%%%%%%%%%%%%%%%%%%%%%%%%%%%%%%%%%%%%%%%%%%%%%%%%%%%%%%%%%%%%%%%%%%

%%% Use this command to include your bibliography file.

\newpage

\begin{ack}
The authors thank the NVIDIA Academic Hardware Grant Program for supporting their experiments, the Australian Research Council (DE230101116 and DP240103070) for Xu, and the Australian Research Council Discovery Projects (DP240103070), and Australian Research Council ARC Early Career Industry Fellowship (IE240100275) for Chen.
\end{ack}

\bibliography{ecai25}

\appendix
\onecolumn

\section{Appendix}

\setcounter{theorem}{0}

\subsection{Related Work}

\subsubsection{Machine Unlearning}

Machine unlearning requires the removal of information about the forgetting data from the original model while preserving the knowledge contained in the remaining data~\cite{sisa,acmsurvey}. Current research on machine unlearning can be categorized into two primary branches based on unlearning requests: \textit{batch unlearning} and \textit{streaming unlearning}. 
Batch unlearning focuses on removing a specific data group within the same batch~\cite{sisa,unroll,ts,boundary,scrub,camu,laf}. This approach typically requires access to the original training data and fine-tuning it to maintain high prediction performance. For instance, \cite{sisa} proposes retraining the model using small data shards from the remaining dataset and ensembling the final results for increased efficiency. Similarly, \cite{unroll} performs incremental training with the forgetting data in the first batch, recording gradients during the initial batch and adding these recorded gradients to the weights after incremental training. 
In contrast, streaming unlearning addresses continuous data removal requests~\cite{ada,onlinelr,dtd,langevin}. For example, \cite{ada} extends \cite{sisa} to be more adaptive to incremental and decremental learning requests in a streaming context. \cite{dtd} proposes a perturbed gradient-descent algorithm on data partitions to update models for streaming unlearning requests. \cite{langevin} fine-tunes the model with noisy gradients for unlearning, which can be extended to streaming unlearning with limited error increases. 
However, these approaches still face limitations. Some are restricted to convex loss functions~\cite{dtd}, while others still rely on full training data access and retraining throughout the unlearning process~\cite{sisa,langevin}. These divergent methodologies underscore the challenges of efficiently applying machine unlearning across various data types and model structures~\cite{ada,onlinelr,dtd,langevin}.

\subsubsection{Online Learning}

Online learning focuses on the task of learning from a continuous data stream, with the minimization of regret risk as the primary objective~\cite{onlinesurvey}. As the size of the incoming data increases, the training data distribution may undergo significant shifts. Consequently, effectively and efficiently adapting to the shifted distribution becomes one of the main challenges in online learning. Among these studies, extensions of regret risk have become an essential area of research. The basic regret compares the cumulative risks of models at each unlearning step with the global optimal model obtained after processing all data in the stream~\cite{ogd}. Adaptive regret divides the entire data stream into smaller time windows and compares the cumulative risks with the optimal models within these windows~\cite{adaregret}. Dynamic regret directly compares the updated models with the optimal ones for each learning request~\cite{dynamicregret,dynamicregret1,dynamicregret2}. Optimizing regret risk is critical to determining the effectiveness and efficiency of online learning algorithms. Specifically, optimization methods include Online Gradient Descent (OGD)~\cite{ogd} for first-order optimization and Online Newton Step (ONS)~\cite{ons} for second-order optimization. Additionally, Online Mirror Descent (OMD)~\cite{OMD} is a common approach that generalizes OGD by performing updates in the dual space, which can be transformed through a regularizer. It is worth noting that under different types of regret, various optimization methods achieve different rigorous error bounds. Finally, in practical problems such as label shift, several online learning algorithms have been proposed~\cite{onlinelds1,onlinelds2,onlinelds3}. These approaches connect label shifts with online learning by continuously updating classification margins. While these algorithms address learning tasks with data streams, there remains a significant research gap in developing online algorithms specifically for streaming unlearning requests.

% Streaming unlearning differs from online learning in the following aspects: Firstly, in machine unlearning, the initial model already has comprehensive knowledge about all the training data, including the data to be forgotten. Streaming unlearning aims to remove the information about the forgotten data, whereas online learning focuses on learning from newly incoming data. This makes the unlearning process inherently more challenging than the learning process. Secondly, the availability of training data in online learning and streaming unlearning differs. In streaming unlearning, the size of the remaining data progressively decreases, and the models must adapt to new optimal states based on this remaining data. However, streaming unlearning does not provide access to the remaining data and must be addressed based on the provided forgetting data. Thirdly, from a practical standpoint, our streaming unlearning method, SAFE, does not assume any convexity or pseudo-convexity for the training loss. We only assume bounded weights and gradients on the training data. These assumptions are more practical as the initial model in unlearning problems has been well-trained, and the gradients on the current training data have stabilized.

\subsection{Notation}\label{appendix:notion}

We provide a table of all notations of the main paper in Table \ref{notiontable}.

\begin{table}[H]
\footnotesize
    \begin{center}
    \caption{Table of Notation}
    \vspace{-2mm}
    \label{notiontable}
    \begin{tabular}{p{1.5cm} p{8cm} }
    \toprule
    Notation &  Explanation\\
    \midrule
    $\mathsf{A}$ &  Learning algorithm\\
    $\mathsf{M}$ &  Unlearning algorithm\\
    $t$ &  Round number of unlearning request\\
    $T$ &  Total length of unlearning request\\
    $D_{0}$ &  Original training data\\
    $D_{t}$ &  Remaining data in the $t$-th round of request\\
    $F_{t}$ &  Forgetting data in the $t$-th request\\
    $f(\cdot; \cdot)$ &  Prediction model\\
    $\mathcal{L}(\cdot,\cdot)$ &  Training loss\\
    $\mathcal{L}_{re}(\cdot,\cdot)$ &  Remaining data loss\\
    $\mathcal{L}_{fr}(\cdot,\cdot)$ &  Forgetting data loss\\
    $\ell(\cdot,\cdot)$ &  Classification loss function\\
    $d_{\text{KL}}(\cdot,\cdot)$ &  Cosine similarity loss\\
    $w_t$ &  Updated parameters in the $t$-th round of request\\
    $w_t^*$ &  Optimal parameters in the $t$-th round of request\\
    $\boldsymbol{\mu}_t$, $\boldsymbol{\Sigma}_t$ &  Mean and covariance matrix
estimated Gaussian distribution of remaining data of class $y$ in the $t$-th round of request\\
    $\nabla(\cdot)$ &  Gradient of the function\\
    $||\cdot||$ & L2 norm\\
    $\gamma$ &  Learning rate in each update\\
    $\epsilon$ &  Mean value for Gaussian perturbation\\
    $\delta$ &  Variance value for Gaussian perturbation\\
    $W$ &  Upper bound of the norm of model parameters\\
    $U$ &  Upper bound of the norm of gradients\\
    \bottomrule
    \end{tabular}
    \end{center}  
\end{table}

\subsection{Theoritical Proof}\label{appendix:tp}

\setcounter{theorem}{0} 

\begin{lemma}(\textbf{Berry-Esseen theorem}) Let $\mathbf{X}_1, \mathbf{X}_2, \ldots, \mathbf{X}_n$ be independent and identically distributed random vectors in $\mathbb{R}^d$ with mean vector $\mathbf{\mu}$ and covariance matrix $\mathbf{\Sigma}$. The Berry-Esseen theorem in the multivariate case states that the upper bound of the error between the real distribution and normalized Gaussian distribution is:
\begin{align}
  \sup_{\mathbf{z} \in \mathbb{R}^d} \left| \mathbb{P}\left(\mathbf{S}_n \leq \mathbf{z}\right) - \Phi_{\mathbf{\Sigma}}(\mathbf{z}) \right| \leq O(\frac{1}{\sqrt{n}}),  ,  
\end{align}
where $\mathbf{S}_n = \frac{1}{\sqrt{n}} \sum_{i=1}^{n} (\mathbf{X}_i - \mathbf{\mu})$ is the normalized sum of the random variables, $\Phi_{\mathbf{\Sigma}}(\mathbf{z})$ is the cumulative distribution function of the multivariate normal distribution with mean vector $\mathbf{0}$ and covariance matrix $\mathbf{\Sigma}$, and $\|\mathbf{X}_1 - \mathbf{\mu}\|$ denotes the Euclidean norm.

\end{lemma}

\setcounter{theorem}{0} 

\begin{theorem}\label{risktheorem2}
If $p_0(y|\mathbf{x}) = f(\mathbf{x}; w_0)$, then $|\widetilde{R}_{t}(w) - R_{t}(w)| \leq C \cdot \sum_{i = 1}^{t} |\mathcal{F}_{i}| / |\mathcal{D}_{t}|^{3/2}$, where $C$ is the number of classes.
\end{theorem}
\begin{proof}
    The estimated predictions of data $(\mathbf{x},y)$ for both the remaining data and forgetting data in the $t$-th round of request is:  
    \begin{align*}
    \widetilde{f}_{t}(\mathbf{x},w_0) = p_t(y|\mathbf{x}) \propto z(\mathbf{x}|y)p_0(y|\mathbf{x}),
    \end{align*}where $z(\mathbf{x}|y)$ is the probability of the estimated Gaussian distribution. In the above equation, $z(\mathbf{x}|y)$ involves bias because the low-dimensional vectors cannot always fit the multivariate Gaussian distribution perfectly. Therefore, based on the Berry-Esseen theorem, we estimate the error between the real distribution and the approximate Gaussian distribution:
     \begin{align*}
    \sup_{(\mathbf{x},y)\in D_{t}}|z(\mathbf{x}|y) - \Phi(n(\mathbf{x}))|\leq O(\frac{1}{\sqrt{|D_{t}|}}),
    \end{align*}where $|D_{t}|$ is the size of $D_{t}$. Therefore, for all the estimated predictions of data $(\mathbf{x},y)$ through distribution shift, one error term exists between the estimated predictions and the optimal predictions:
     \begin{align*}
    \sup_{(\mathbf{x},y)\in D_{t}}|\widetilde{f}_{t}(\mathbf{x},w_0) - f(\mathbf{x},w^*_t)|\leq O(\frac{1}{\sqrt{|D_{t}|}}).
    \end{align*}

    Therefore, for any historical sequential unlearning requests $\{F_{t}\}$, the estimated risk of $\hat{R}_{t}(w)$ after removing $F$ can be represented by:

    \begin{align*}
    |\widetilde{R}_{t}(w) - R_{t}(w)| = & \frac{1}{|D_{t}|}\sum_{i=1}^{t}\sum_{(\mathbf{x},y)\in F_{i}}(f(\mathbf{x}; w)|\log(q_t^{(y)}(\mathbf{x})f(\mathbf{x}; w_0)) - \log(f(\mathbf{x}; w_t^*)))|\\
    \leq & \frac{1}{|D_{t}|}\sum_{i=1}^{t}\sum_{(\mathbf{x},y)\in F_{i}}(f(\mathbf{x}; w)|\frac{1}{q_t^{(y)}(\mathbf{x})f(\mathbf{x}; w_0)} - \frac{1}{f(\mathbf{x}; w_t^*)})|\\
    = & \frac{1}{|D_{t}|}\sum_{i=1}^{t}\sum_{(\mathbf{x},y)\in F_{i}}(\frac{f(\mathbf{x}; w)}{q_t^{(y)}(\mathbf{x})f(\mathbf{x}; w_0)f(\mathbf{x}; w_t^*)}|q_t^{(y)}(\mathbf{x})f(\mathbf{x}; w_0) - f(\mathbf{x}; w_t^*))|\\
    \leq & \frac{C\sum_{i = 1}^{t}|F_{i}|}{|D_{t}|^{\frac{3}{2}}},
    \end{align*}
    where the total size of forgetting data $\sum_{i = 1}^{t}|F_{i}|$ is always less than the size of remaining data $|D_{t}|$ and $C$ is the count of classes and $\sqrt{|D_{t}|}\gg \sum_{i = 1}^{t}|F_{i}|$, and it demonstrates that $|\widetilde{R}_{t}(w) - R_{t}(w)|\approx 0$.
\end{proof}

\setcounter{theorem}{1} 

% \begin{lemma}~\cite{analysisda} Let $\mathcal{R}$ be a fixed representation function from $\mathcal{X}$ to $\mathcal{Z}$ and $\mathcal{H}$ be a hypothesis space of VC-dimension $d$. If a random labeled sample of size $m$ is generated by applying $\mathcal{R}$ to a $D_S$-i.i.d. sample labeled according to $f$, then with probability at least $1-\delta$, for every $h \in \mathcal{H}$:
% \[
% \epsilon_T(h) \leq \hat{\epsilon}_S(h) + \sqrt{\frac{4}{m} \left( d \log \frac{2em}{d} + \log \frac{4}{\delta} \right)} + \text{div}(D_S, D_T)
% \]
% where $e$ is the base of the natural logarithm, and $\text{div}(\tilde{D}_S, \tilde{D}_T)$ is the distance between source domain data $D_S$ and target domain data $D_T$.
% \end{lemma}

% The proof can be found in the proof of theorem 1 in~\cite{analysisda}

\setcounter{theorem}{1} 

\begin{theorem}\label{theorem:th5}
Suppose the risk $R_{t}(w)$ has gradients upper bounded by $U$, i.e., $\|\nabla R_{t}(w)\|_2 \leq U$, and the model parameters satisfy $\|w\|_2 \leq W$. Set the learning rate $\gamma = \frac{\sqrt{W}}{K\sqrt{T}}$ and the perturbation variance $\phi_t = \frac{W\sqrt{2\ln{(1.25/\delta)}}}{\epsilon}$, where $K$ is a constant. Let $w_t$ denote the output of Algorithm 1 at step $t$. Then, the following hold:

\noindent(i) The expected error at step $t$ compared to the optimal model $w_t^*$ is upper bounded:
\begin{align*}
 \mathbb{E}\left[R_{t}(w_t) - R_{t}(w_t^*)\right] &\leq O(\sqrt{T}).
\end{align*}
\end{theorem}
\begin{proof}
    Let $w_t^*$ denote the optimal model parameters in the $t$-th round of removal request and $\nabla_t = \frac{g_t}{\|g_t\|_2} - \frac{g_{t-1}}{\|g_{t-1}\|_2}$ in each optimization steps. Then the difference between $w_t$ and $w_t^*$ is:
    \begin{align}\label{theorem5:aeq1}
    ||w_t - w^*_t||^2 
    = &||w_t - w^*_t - w^*_{t-1} + w^*_{t-1}||^2 = ||w_t -  w^*_{t-1} ||^2 +|| w^*_t - w^*_{t-1}||^2 - 2(w_t -  w^*_{t-1})^\top(w^*_t - w^*_{t-1}).
    \end{align}
    
    After incorporating $w_t$ into the first item of eq~\ref{theorem5:aeq1} as shown in the following:
    \begin{align}\label{theorem5:aeq2}
    ||w_t -  w^*_{t-1}||^2  
    =&||w_{t-1} -  \gamma \nabla_t - b_t + b_{t-1} -  w^*_{t-1} ||^2\notag \\ 
    = &||w_{t-1} -  w^*_{t-1}||^2 + ||\gamma \nabla_t||^2 - 2\gamma (\nabla_t)^\top(w_{t-1} -  w^*_{t-1}) + \\&||b_t - b_{t-1}||^2 - 2(b_t - b_{t-1})^\top(w_{t-1} -  \gamma \nabla_t - b_t + b_{t-1} -  w^*_{t-1}) .
    \end{align}

    We can incoperate Eq.~\ref{theorem5:aeq2} into Eq.~\ref{theorem5:aeq1}:

    \begin{align}\label{theorem1:eq3}
    ||w_t - w^*_t||^2 
    = &|| w^*_t - w^*_{t-1}||^2 + ||w_{t-1} -  w^*_{t-1}||^2 + ||\gamma \nabla_t||^2 - 2\gamma (\nabla_t)^\top(w_{t-1} -  w^*_{t-1}) - \notag \\ 
    &  2(w_t -  w^*_{t-1})^\top(w^*_t - w^*_{t-1}) + B_t,
    \end{align}where $B_t = ||b_t - b_{t-1}||^2 - 2(b_t - b_{t-1})^\top(w_{t-1} -  \gamma \nabla_t - b_t + b_{t-1} -  w^*_{t-1})$
        
    By rearranging terms and multiplying $\frac{1}{2\gamma}$ on both sides we have:
    
    \begin{align*}
    &(\nabla_t)^\top(w_{t-1} -  w^*_{t-1})\\
    = & \frac{1}{2\gamma}[||w_{t-1} - w^*_{t-1}||^2 - ||w_{t} -  w^*_{t}||^2 + ||\gamma \nabla_t||^2 + (w^*_t +  w^*_{t-1} - 2w_{t})^\top(w^*_t - w^*_{t-1}) + B_t]\\
    = & \frac{1}{2\gamma}[||w_{t-1} - w^*_{t-1}||^2 - ||w_{t} -  w^*_{t}||^2 + ||\gamma \nabla_t||^2 + ||w^*_t||^2 -  ||w^*_{t-1}||^2  - 2w_{t}^\top(w^*_t - w^*_{t-1}) + B_t]\\
    \le & \frac{1}{2\gamma}[2W|| w^*_t - w^*_{t-1}|| +||w_{t-1} - w^*_{t-1}||^2 - ||w_{t} -  w^*_{t}||^2 + \gamma^2 + ||w^*_t||^2 -  ||w^*_{t-1}||^2 + B_t].
    \end{align*}

    Therefore the estimated error of training loss $f$ in the $t$ round is:
    \begin{align*}
        \mathbb{E}|R_{t}(w_{t}) - R_{t}(w^*_{t})|
        \le & U\mathbb{E}\left[|(\nabla_{t+1})^\top(w_t - w^*_t)\right] \\
        \le & \frac{U}{2\gamma}\mathbb{E}\left[2W|| w^*_{t+1} - w^*_{t}|| +||w_{t} - w^*_{t}||^2 - ||w_{t+1} -  w^*_{t+1}||^2 + \gamma^2 + ||w^*_{t+1}||^2 -  ||w^*_{t}||^2 + B_t\right] \\
        \le & \frac{U}{2\gamma}\left[4W^2 + 4W^2 + \gamma^2 + W^2\right] = \frac{9K}{2}UW^{\frac{3}{2}}\sqrt{T} + U\frac{\sqrt{W}}{K\sqrt{T}} = O(\sqrt{T}).
    \end{align*}\end{proof}

\setcounter{theorem}{2}

\begin{theorem}\label{theorem:th6}
Suppose the risk $R_{t}(w)$ has gradients upper bounded by $U$, i.e., $\|\nabla R_{t}(w)\|_2 \leq U$, and the model parameters satisfy $\|w\|_2 \leq W$. Set the learning rate $\gamma = \frac{\sqrt{W}}{K\sqrt{T}}$ and the perturbation variance $\phi_t = \frac{W\sqrt{2\ln{(1.25/\delta)}}}{\epsilon}$, where $K$ is a constant. Let $w_t$ denote the output of Algorithm 1 at step $t$. Then, the following hold:

\noindent(ii) The accumulated unlearning regret across all $T$ steps is upper bounded:
\begin{align*}
 \mathbb{E}\left[\sum_{t=1}^{T} \left( R_{t}(w_t) - R_{t}(w_t^*) \right) \right] &\leq O(\sqrt{T} + V_T),
\end{align*}
where $V_T = \sum_{t=1}^{T} \| w_t^* - w_{t-1}^* \|_2$.
\end{theorem}

\begin{proof}

    We can get the upper bound of the regret estimation by summing the upper bound of $\mathcal{L}(D_{t}, w_t) - \mathcal{L}(D_{t}, w_t^*)$ for each $t = 1,\ldots, T$:
    \begin{align}\notag
    &\mathbb{E}[\sum_{t=1}^{T}|R_{t}(w_{t}) - R_{t}(w^*_{t})|] \notag\\
    \le& \sum_{t=0}^{T-1} U\mathbb{E}[(\nabla_{t+1})^\top(w_t - w^*_t)] \notag\\
    = & \frac{U}{2\gamma}\sum_{t=1}^{T}\left[||w_{t-1} - w^*_{t-1}||^2 - ||w_{t} -  w^*_{t}||^2 + ||\gamma||^2 + (w_{t-1} - w_{t})^\top(w^*_t - w^*_{t-1}) + ||w^*_t||^2 -  ||w^*_{t-1}||^2  - 2w_{t}^\top(w^*_t - w^*_{t-1})\right]\notag\\
    = & \frac{U}{2\gamma}\sum_{t=1}^{T}[||w_{t-1} - w^*_{t-1}||^2 - ||w_{t} -  w^*_{t}||^2 + ||\gamma||^2 + (w_{t-1} - w_{t})^\top(w^*_t - w^*_{t-1}) + ||w^*_t||^2 -  ||w^*_{t-1}||^2+\notag\\
    & 2 \gamma(\frac{g_t}{||g_t||_2})^\top(w^*_t - w^*_{t-1}) - 4w_0^\top(w^*_t - w^*_{t-1})]\notag\\
    < & \frac{U}{2\gamma}(\sum_{t=1}^{T}\left[||\gamma||^2 + (w_{t-1} - w_{t})^\top(w^*_t - w^*_{t-1}) + 2 \gamma(\frac{g_t}{||g_t||_2})^\top(w^*_t - w^*_{t-1}) - 4w_0^\top(w^*_t - w^*_{t-1})\right] + 5W^2)\notag\\
    < & \frac{U}{2\gamma}(\sum_{t=1}^{T}\left[||\gamma||^2 + \gamma(\nabla_{t})^\top(w^*_t - w^*_{t-1}) + 2 \gamma(\frac{g_t}{||g_t||_2})^\top(w^*_t - w^*_{t-1}))\right] + 9W^2)\notag\\
    = &\frac{U}{2\gamma}(\sum_{t=1}^{T}\left[\gamma(\nabla_{t})^\top(w^*_t - w^*_{t-1}) + 2 \gamma(\frac{g_t}{||g_t||_2})^\top(w^*_t - w^*_{t-1}))\right] + T||\gamma||^2  + 9W^2)\notag\\
    = & \frac{UV_T}{2} + \frac{TU\gamma}{2} + \frac{9UW^2}{2\gamma} = O(\sqrt{T} + V_T)\notag
    \end{align} 
\end{proof}

% \begin{proof}

%     We can get the upper bound of the regret estimation by summing the upper bound of $\mathcal{L}(D_{t}, w_t) - \mathcal{L}(D_{t}, w_t^*)$ for each $t = 1,\ldots, T$:
%     \begin{align}\notag
%     &\mathbb{E}[\sum_{t=1}^{T}|R_{t}(w_{t}) - R_{t}(w^*_{t})|] \notag\\
%     \le& \sum_{t=0}^{T-1} U\mathbb{E}[|(\nabla_{t+1})^\top(w_t - w^*_t)] \notag\\
%     = & \frac{U}{2\gamma}\sum_{t=1}^{T}\left[||w_{t-1} - w^*_{t-1}||^2 - ||w_{t} -  w^*_{t}||^2 + ||\gamma||^2 + (w_{t-1} - w_{t})^\top(w^*_t - w^*_{t-1}) + (w^*_t +  w^*_{t-1} - w_{t}- w_{t-1})^\top(w^*_t - w^*_{t-1})\right]\notag\\
%     = & \frac{U}{2\gamma}\sum_{t=1}^{T}\left[||\gamma||^2 + (w_{t-1} - w_{t})^\top(w^*_t - w^*_{t-1}) - (w^*_t +  w^*_{t-1} - w_{t}- w_{t-1})^\top(w_{t-1} - w_{t})\right]\notag\\
%     = & \frac{U}{4\gamma}\sum_{t=1}^{T}\left[2||\gamma||^2 + 2(w_{t-1} - w_{t})^\top(w^*_t - w^*_{t-1}) + ||w_{t-1} - w^*_{t-1}||^2 - ||w_{t} -  w^*_{t}||^2 +  ||w^*_{t} - w_{t-1}||^2 - ||w^*_{t-1} - w_{t}||^2\right]\notag\\
%     < & \frac{U}{2\gamma}(\sum_{t=1}^{T}\left[||\gamma||^2 + (w_{t-1} - w_{t})^\top(w^*_t - w^*_{t-1}) + (w^*_{t-1})^\top (w_{t}) - (w^*_{t})^\top (w_{t-1})\right] + 3W)\notag\\
%     < & \frac{U}{4\gamma}(\sum_{t=1}^{T}\left[2||\gamma||^2 + (w_{t-1} - w_{t})^\top(w^*_t - w^*_{t-1}) + (w^*_{t-1})^\top (w_{t-1}) - (w^*_{t})^\top (w_{t})\right] + 6W)\notag\\
%     = & \frac{U}{2\gamma}(\sum_{t=1}^{T}\left[||\gamma||^2 + (w_{t-1} - w_{t})^\top(w^*_t - w^*_{t-1})\right]+ 8W) = KUV_T + \frac{TU\gamma}{2} + \frac{4UW}{\gamma} = O(\sqrt{T} + V_T)\notag
%     \end{align} 
% \end{proof}

\setcounter{theorem}{3}

\begin{theorem}\label{theorem:th7}
Suppose the risk $R_{t}(w)$ has gradients upper bounded by $U$, i.e., $\|\nabla R_{t}(w)\|_2 \leq U$, and the model parameters satisfy $\|w\|_2 \leq W$. Set the learning rate $\gamma = \frac{\sqrt{W}}{K\sqrt{T}}$ and the perturbation variance $\phi_t = \frac{W\sqrt{2\ln{(1.25/\delta)}}}{\epsilon}$, where $K$ is a constant. Let $w_t$ denote the output of Algorithm 1 at step $t$. Then, the following hold:

\noindent(iii) The unlearning at each step satisfies $(\epsilon, \delta)$-\textit{approximate unlearning}.
\end{theorem}
\begin{proof}
The proof follows along the proof of the differential privacy guarantee for the Gaussian mechanism in Appendix A in~\cite{dpbook}
\end{proof}

\section{Experiments}

In this section, we provide a detailed description of the baseline methods and the details of the implementation of the streaming unlearning algorithm. We then present additional experimental results to address the following five research questions, which are crucial for evaluating the streaming unlearning algorithm:

\begin{itemize}
    \item \textbf{RQ1}: How does SAFE perform on tabular and text datasets compared to other methods?
    \item \textbf{RQ2}: How does streaming unlearning perform under different settings of rounds and forgetting data sizes compared to other methods?
    \item \textbf{RQ3}: How does the hyperparameter $K$ affect the unlearning performance?
    \item \textbf{RQ4}: How does SAFE work in streaming unlearning for a specific class?
\end{itemize}

\subsection{Baselines}\label{appendix:baselines}

We compare the performance of the SAFE algorithm with \textbf{Retrain}, which represents the standard results from retrained models, as well as four state-of-the-art unlearning methods with high efficiency and potential to handle sequential unlearning requests: Lcode~\cite{mehta2022deep}, Desc~\cite{dtd}, Unrolling~\cite{unroll}, and CaMU~\cite{camu}. \textbf{Lcode}~\cite{mehta2022deep} applies a pruning strategy to select model parameters associated with the selected forgetting data and then uses a gradient ascent algorithm, as described in~\cite{RememberWhatYouWanttoForget}, to unlearn the selected parameters. \textbf{Desc}~\cite{dtd} computes the gradients of the remaining data and applies perturbed gradient descent for unlearning. \textbf{Unroll}~\cite{unroll} records gradients during the first epoch of training and adds these recorded gradients to the weights after incremental training. \textbf{CaMU} constructs counterfactual samples for each forgetting sample and performs unlearning at both the representation and prediction levels.

\subsection{Implementation Details}\label{appendix:implementation}

All the experiments are conducted on one server with NVIDIA RTX A5000 GPUs (24GB GDDR6 Memory) and 12th Gen Intel Core i7-12700K CPUs (12 cores and 128GB Memory). The code of SAFE was implemented in Python 3.9.16 and Cuda 11.6.1. The main versions of Python packages are Numpy 1.23.5, Pandas 2.0.1, Pytorch 1.13.1, and Torchvision 0.14.1. 

All the experiments on these baselines are conducted under 10 random seeds based on the original models trained in the four datasets. We train two CNN models on MNIST and Fashion datasets for 20 epochs with a learning rate of 1e-3 and a weight decay of 1e-4. We train another ResNet-18 model on the CIFAR10 dataset for 20 epochs, where the learning rate is set as 0.1 and other hyperparameters are the same as the code in \hyperlink{}{https://github.com/kuangliu/pytorch-cifar/tree/master}. For TinyImagenet, we use 80\% data as training data and train them in for 20 epochs using the same setting as the training process on CIFAR10 dataset. For the two MNIST datasets, the batch size is set as 32, and for the other two datasets, the batch size is 128. For the retrained models, we adopt the same hyperparameters as the training process of the original model.

Then, the hyperparameters used in the implementation of SAFE only include the amplification factor of the distribution shift loss and the learning rate. For the amplification factor, we set it as 2000 for the MNIST dataset, 6000 for the Fashion dataset, 100000 for the CIFAR10 dataset, and 1000000 for the TinyImagenet dataset. Then for the hyperparameter $K$, we set it as 2.5, 8, 2.5 and 4 for each dataset.

For the evaluations, we assess the unlearning algorithm using five metrics: $\textbf{RA}$, $\textbf{FA}$, and $\textbf{TA}$, which denote the prediction accuracy of the post-unlearning model on the remaining data, forgetting data, and test data. The closer value to the retrained model indicates better unlearning performance for these metrics. We also check the attack accuracy of the $\textbf{MIA}$~\cite{mia,miainu}. Specifically, we choose the same MIA evaluation as~\cite{salun,DBLP:conf/nips/JiaLRYLLSL23}. Specifically, we use the subset of remaining data with the size of 10000 as positive data and real test data with the size of 10000 as negative data to construct the attacker model’s training set. Then, we train an SVC model with the Radial Basis Function Kernel model as the attacker. Then, the attacker was evaluated using the forgetting data to measure attack success rates.

\subsection{Additional Experiment Results}\label{appendix:additional}

To demonstrate SAFE’s applicability beyond image data, we conducted additional experiments on a text dataset (SST-2) and a tabular dataset (CoverType). For SST-2, we use pre-trained DistilBert to extract text embeddings and use a logistic regression model for classification. For CoverType, we use a logistic regression model for classification. For both datasets, we select 80\% data for training and 20\% data for testing. We conduct 20 rounds of unlearning, and each round contains 400 forgetting data points. We present the results in the two tables below. SAFE can always achieve the highest average ranks (R) across all evaluations. SAFE can obtain the forgetting data performance (FA) and top-2 predictive ability (TA).

\begin{table}[htbp]\footnotesize
\renewcommand{\arraystretch}{1.2}
	\centering
	\caption{Comparison results in Random Subset Unlearning (avg\%$\pm$std\%). The \textbf{bold} record indicates the best, and the \underline{underlined} record indicates the second-best. The white backgrounds denote the unlearning methods that require remaining data in experiments, while the grey backgrounds denote the approaches without remaining data.}
        \vspace{6mm}
        {\fontsize{6.5}{7.5}\selectfont
	\begin{tabular}{p{0.78cm}<{\centering} | p{1.26cm}<{\centering} p{1.27cm}<{\centering} p{1.26cm}<{\centering} p{1.26cm}<{\centering} p{0.35cm}<{\centering} | p{1.26cm}<{\centering} p{1.27cm}<{\centering} p{1.26cm}<{\centering} p{1.26cm}<{\centering} p{0.35cm}<{\centering} }
		\toprule%\hline    
            \multirow{2}{*}{\textbf{Method}} &\textbf{RA} &\textbf{FA} &\textbf{TA}  &\textbf{MIA}&\textbf{R}&\textbf{RA} &\textbf{FA} &\textbf{TA}  &\textbf{MIA}&\textbf{R}\\

            ~ &\multicolumn{5}{c}{\textbf{CoverType}} &\multicolumn{5}{c}{\textbf{SST-2}} \\ 
            
            \midrule%\hline
            \textbf{Retrain}& 59.54±0.03 & 40.13±0.06 & 59.38±0.03 & 69.57±8.48 & - & 87.42±0.01 & 86.23±0.01 & 86.18±0.01 & 52.04±0.01 & - \\

            \textbf{Unroll}&  45.76±0.04 & 22.21±0.04 & 45.44±0.04 & \underline{58.80±19.21} & 3.5 &  \textbf{87.41±0.01} & 87.26±0.01 & \textbf{86.20±0.01} & 52.98±0.05 & \underline{2.75}\\

            \textbf{CaMU}& 70.87±0.01 & \underline{47.56±0.02} & 70.49±0.01 & \textbf{74.34±0.01} & \underline{2.25} &  \underline{87.39±0.01} & 87.33±0.01 & \underline{86.11±0.01} & 52.66±0.05 & 3\\
            \rowcolor{gray!20}
            
            \textbf{Lcode} & \underline{50.99±0.07} & \underline{32.70±0.06} & \underline{50.78±0.07} & 53.20±15.14 & 2.5 &  87.20±0.01 & \underline{87.07±0.01} & 86.07±0.01 & \underline{52.50±0.06} & 3.25\\
            
            \rowcolor{gray!20}
            \textbf{Desc}& 39.28±0.10 & 14.85±0.03 & 38.92±0.10 & 48.97±30.85 & 5 & 87.26±0.01 & 87.16±0.01 & \underline{86.11±0.01} & 52.67±0.06 & 3.25\\
            
            \rowcolor{gray!40}
            \textbf{SAFE}& \textbf{57.30±0.02} & \textbf{41.68±0.20} & \textbf{56.91±0.02} & 82.14±3.03 & \textbf{1.5} &  87.33±0.01 & \textbf{86.59±0.01} & 85.87±0.01 & \textbf{52.06±0.01} & \textbf{2.5}\\
                        
            \bottomrule

	\end{tabular}\label{addcomparison}
    }
    \end{table}

\subsection{Further Comparison Under Different Settings}\label{appendix:settingcompare}

In the following four tables: Table~\ref{settingcomparemnist}, Table~\ref{settingcomparemnistfashion}, Table~\ref{settingcomparecifar10feature}, and Table~\ref{settingcomparecifar}, we present additional experimental results under different unlearning request settings. For each dataset, we conduct four groups of experiments with varying configurations:
\vspace{-1mm}
\begin{itemize}
\setlength\itemsep{1mm}
    \item Setting the unlearning round to 10 and removing 400 samples in each round.
    \item Setting the unlearning round to 40 and removing 400 samples in each round. 
    \item Setting the unlearning round to 10 and removing 800 samples in each round.
    \item Setting the unlearning round to 20 and removing 800 samples in each round.
\end{itemize}
\begin{table}[H]\footnotesize
\renewcommand{\arraystretch}{1.2}
	\centering
	\caption{Effect analysis on size and rounds of forgetting requests on MNIST (avg\%$\pm$std\%).}
        \vspace{6mm}
        {\fontsize{6.5}{7.5}\selectfont
	\begin{tabular}{p{0.95cm}<{\centering} | p{2cm}<{\centering} p{2cm}<{\centering} p{2cm}<{\centering}| p{2cm}<{\centering} p{2cm}<{\centering} p{1.4cm}<{\centering} }
		\toprule%\hline    
            \multirow{2}{*}{\textbf{Method}} &\textbf{RA} &\textbf{FA} &\textbf{TA}  &\textbf{RA} &\textbf{FA} &\textbf{TA} \\
        
        ~ &\multicolumn{3}{c|}{\textbf{400 Samples for 10 Rounds}} & \multicolumn{3}{c}{\textbf{800 Samples for 10 Rounds}} \\ 
        
        \midrule%\hline
        {Retrain}& 99.69$\pm$0.04& 98.94$\pm$0.08& 98.99$\pm$0.05& 99.70$\pm$0.03& 98.78$\pm$0.05& 98.99$\pm$0.05\\

        Unroll& 99.80$\pm$0.04& 99.60$\pm$0.05& 99.08$\pm$0.03& \textbf{99.76$\pm$0.03}& 99.51$\pm$0.09& 99.04$\pm$0.03\\
        
        CaMU& 98.90$\pm$0.52& \underline{98.50$\pm$1.12}& 98.38$\pm$0.57& 98.94$\pm$0.15& \underline{98.55$\pm$0.35}& 98.55$\pm$0.21\\

        \rowcolor{gray!20}
        Lcode& 29.34$\pm$5.34& 29.31$\pm$5.13& 29.54$\pm$5.36& 27.82$\pm$3.64& 27.90$\pm$3.23& 28.02$\pm$3.72\\

        \rowcolor{gray!20}
        Desc& \textbf{99.64$\pm$0.04}& 99.62$\pm$0.04& \textbf{98.98$\pm$0.04}& \underline{99.64$\pm$0.04}& 99.59$\pm$0.04& \textbf{98.98$\pm$0.04}\\        
        
        \rowcolor{gray!40}
        
        SAFE& \underline{99.36$\pm$0.12}& \textbf{99.10$\pm$0.32}& \underline{98.70$\pm$0.20}& 99.35$\pm$0.10&\textbf{99.09$\pm$0.19}& \underline{98.72$\pm$0.08}\\
        
        \midrule
        ~ &\multicolumn{3}{c|}{\textbf{400 Samples for 40 Rounds}} & \multicolumn{3}{c}{\textbf{800 Samples for 20 Rounds}} \\ 
        \midrule
        {Retrain}& 99.69$\pm$0.04& 98.84$\pm$0.09& 98.97$\pm$0.06& 99.69$\pm$0.03& 98.79$\pm$0.05& 98.97$\pm$0.06\\

        Unroll& 98.78$\pm$0.05& \underline{98.72$\pm$0.04}& 98.27$\pm$0.04& \textbf{99.74$\pm$0.04}& 99.54$\pm$0.07& \underline{99.01$\pm$0.03}\\
        
        CaMU& 98.72$\pm$0.36& 98.50$\pm$0.66& 98.47$\pm$0.32& 98.69$\pm$0.33& \textbf{98.34$\pm$0.37}& 98.44$\pm$0.24\\

        \rowcolor{gray!20}
        Lcode& 28.78$\pm$4.24& 28.84$\pm$4.15& 28.92$\pm$4.27& 29.49$\pm$4.48& 29.98$\pm$4.51& 29.77$\pm$4.50\\

        \rowcolor{gray!20}
        Desc& \underline{99.40$\pm$0.18}& 99.37$\pm$0.18& \underline{98.79$\pm$0.15}& \underline{99.56$\pm$0.09}& 99.53$\pm$0.07& 98.92$\pm$0.07\\
        
        \rowcolor{gray!40}
        
        SAFE& \textbf{99.52$\pm$0.04}&\textbf{99.04$\pm$0.05}&\textbf{98.86$\pm$0.03}& 99.37$\pm$0.08& \underline{99.12$\pm$0.05}& \textbf{98.73$\pm$0.07}\\
        
        \bottomrule%\hline
        
\end{tabular}\label{settingcomparemnist}
}

\end{table}

\begin{table}[H]\footnotesize
\renewcommand{\arraystretch}{1.2}
	\centering
	\caption{Effect analysis on size and rounds of forgetting requests on Fashion (avg\%$\pm$std\%).}
        \vspace{6mm}
        {\fontsize{6.5}{7.5}\selectfont
	\begin{tabular}{p{0.95cm}<{\centering} | p{2cm}<{\centering} p{2cm}<{\centering} p{2cm}<{\centering}| p{2cm}<{\centering} p{2cm}<{\centering} p{1.4cm}<{\centering} }
		\toprule%\hline    
            \multirow{2}{*}{\textbf{Method}} &\textbf{RA} &\textbf{FA} &\textbf{TA}  &\textbf{RA} &\textbf{FA} &\textbf{TA} \\
        
        ~ &\multicolumn{3}{c|}{\textbf{400 Samples for 10 Rounds}} & \multicolumn{3}{c}{\textbf{800 Samples for 10 Rounds}} \\ 

        \midrule%\hline

        {Retrain}& 96.40$\pm$0.18& 90.82$\pm$0.44& 90.48$\pm$0.14& 96.43$\pm$0.15& 91.15$\pm$0.53& 90.36$\pm$0.14\\
        
        Unroll& 90.92$\pm$0.68& 89.05$\pm$0.94& 88.29$\pm$0.60& 90.68$\pm$0.64& 88.95$\pm$0.96& 88.03$\pm$0.61\\
        
        CaMU& 91.44$\pm$0.34& \textbf{90.39$\pm$0.84}& \underline{88.99$\pm$0.44}& 91.30$\pm$0.25& \underline{90.63$\pm$5.27}& 89.03$\pm$2.23\\

        \rowcolor{gray!20}

        Lcode& 25.65$\pm$4.35& 25.47$\pm$4.79& 25.38$\pm$4.29& 24.49$\pm$3.08& 24.50$\pm$3.80& 24.14$\pm$2.99\\
        \rowcolor{gray!20}
        Desc& \textbf{93.06$\pm$0.30}& 92.97$\pm$0.49& \textbf{90.17$\pm$0.18}& \textbf{93.05$\pm$0.30}& 93.33$\pm$0.63& \textbf{90.17$\pm$0.18}\\
        
        \rowcolor{gray!40}
        
        SAFE& \underline{91.32$\pm$0.27}& \underline{89.60$\pm$0.32}& 88.64$\pm$0.27& \underline{91.78$\pm$0.34}&\textbf{90.68$\pm$0.53}& \underline{89.04$\pm$0.30}\\
        
        \midrule
        ~ &\multicolumn{3}{c|}{\textbf{400 Samples for 40 Rounds}} & \multicolumn{3}{c}{\textbf{800 Samples for 20 Rounds}} \\ 
        \midrule
        {Retrain}& 96.56$\pm$0.19& 90.71$\pm$0.25& 90.24$\pm$0.21& 96.53$\pm$0.20& 90.80$\pm$0.53& 90.21$\pm$0.21\\

        Unroll& 90.54$\pm$0.81& 89.31$\pm$0.96& 87.89$\pm$0.74& 90.79$\pm$0.72& 89.28$\pm$0.94& 88.08$\pm$0.68\\
        
        CaMU& 90.75$\pm$0.71& \underline{89.95$\pm$0.77}& 88.67$\pm$0.52& 90.75$\pm$0.61& \underline{89.97$\pm$0.86}& 88.66$\pm$0.44\\

        \rowcolor{gray!20}

        Lcode& 24.79$\pm$4.13& 25.33$\pm$4.41& 24.57$\pm$4.06& 25.13$\pm$4.52& 25.18$\pm$4.83& 24.78$\pm$4.43\\
        \rowcolor{gray!20}
        
        Desc& \underline{92.22$\pm$0.60}& 92.12$\pm$0.60& \textbf{89.60$\pm$0.42}& \textbf{92.69$\pm$0.46}& 92.75$\pm$0.76& \textbf{89.93$\pm$0.32}\\
        
        \rowcolor{gray!40}
        
        SAFE& \textbf{92.32$\pm$0.22}& \textbf{91.56$\pm$0.25}& \underline{88.82$\pm$0.20}& \underline{92.20$\pm$0.21}&\textbf{91.36$\pm$0.37}& \underline{89.38$\pm$0.17}\\
                
        \bottomrule%\hline
        
\end{tabular}\label{settingcomparemnistfashion}
}
	\vspace{-4mm}
\end{table}

\begin{table}[H]\footnotesize
\renewcommand{\arraystretch}{1.2}
	\centering
	\caption{Effect analysis on size and rounds of forgetting requests on CIFAR10 (avg\%$\pm$std\%).}
        \vspace{6mm}
{\fontsize{6.5}{7.5}\selectfont
\begin{tabular}{p{0.95cm}<{\centering} | p{2cm}<{\centering} p{2cm}<{\centering} p{2cm}<{\centering}| p{2cm}<{\centering} p{2cm}<{\centering} p{1.4cm}<{\centering} }
		\toprule%\hline    
            \multirow{2}{*}{\textbf{Method}} &\textbf{RA} &\textbf{FA} &\textbf{TA}  &\textbf{RA} &\textbf{FA} &\textbf{TA} \\
        
        ~ &\multicolumn{3}{c|}{\textbf{400 Samples for 10 Rounds}} & \multicolumn{3}{c}{\textbf{800 Samples for 10 Rounds}} \\  
        
        \midrule%\hline

        {Retrain}& 97.75$\pm$0.15& 92.15$\pm$0.39& 91.44$\pm$0.22& 97.57$\pm$0.25& 91.48$\pm$0.50& 91.11$\pm$0.41\\
   
        Unroll& \textbf{95.97$\pm$1.23}& \textbf{92.32$\pm$0.67}& \underline{88.95$\pm$1.04}& 92.21$\pm$2.71& 87.88$\pm$1.87& 85.75$\pm$2.28\\
        
        CaMU& \underline{95.10$\pm$1.31}& \underline{91.46$\pm$4.34}& \textbf{89.02$\pm$1.30}&  \textbf{95.62$\pm$0.48}& \underline{93.16$\pm$1.72}& \textbf{89.52$\pm$0.46}\\

        \rowcolor{gray!20}
        Lcode& 24.17$\pm$4.37& 24.43$\pm$4.97& 23.55$\pm$4.17& 22.06$\pm$1.48& 22.16$\pm$1.62& 22.51$\pm$1.44\\
        \rowcolor{gray!20}
        Desc& 94.07$\pm$1.89& 94.42$\pm$2.09& 88.23$\pm$1.64& \underline{94.04$\pm$1.92}& \underline{93.00$\pm$2.12}&  \underline{88.19$\pm$1.58}\\
        
        \rowcolor{gray!40}
        
        SAFE& 91.16$\pm$3.56& 90.89$\pm$3.48& \underline{85.65$\pm$2.96}& 93.15$\pm$1.44&\textbf{91.47$\pm$1.06}& 87.11$\pm$1.18\\
        
        \midrule
        ~ &\multicolumn{3}{c|}{\textbf{400 Samples for 40 Rounds}} & \multicolumn{3}{c}{\textbf{800 Samples for 20 Rounds}} \\ 
        \midrule
        {Retrain}& 97.10$\pm$0.64& 90.97$\pm$0.96& 90.45$\pm$0.87& 97.06$\pm$0.68& 90.65$\pm$1.03& 90.35$\pm$0.93\\
        
        Unroll& 89.34$\pm$4.90& 86.84$\pm$4.16& 83.37$\pm$4.16& 87.77$\pm$5.07& 84.29$\pm$4.05& 81.87$\pm$4.41\\
        
        CaMU& \textbf{96.03$\pm$0.84}& 93.95$\pm$2.68& \textbf{89.84$\pm$0.82}& \textbf{95.73$\pm$0.43}&  \underline{93.63$\pm$1.37}& \textbf{89.56$\pm$0.44}\\

        \rowcolor{gray!20}
        Lcode& 23.55$\pm$3.19& 23.47$\pm$3.53& 22.90$\pm$3.04& 22.70$\pm$2.11& 22.72$\pm$1.98& 22.12$\pm$1.95\\
        \rowcolor{gray!20}
        Desc& 90.75$\pm$9.54& \textbf{90.95$\pm$9.59}& 85.43$\pm$8.40& 72.98$\pm$28.94& 72.86$\pm$28.88& 69.00$\pm$26.72\\

        \rowcolor{gray!40}
        
        SAFE&  \underline{92.26$\pm$3.54}&  \underline{91.13$\pm$3.46}&  \underline{85.80$\pm$2.53}&  \underline{91.42$\pm$2.81}&\textbf{90.84$\pm$2.20}& \underline{85.84$\pm$2.26}\\
        
        \bottomrule%\hline
        
\end{tabular}\label{settingcomparecifar}
}
\end{table}

\begin{table}[H]\footnotesize
\renewcommand{\arraystretch}{1.2}

	\centering
	\caption{Effect analysis on size and rounds of forgetting requests on TinyImangenet (avg\%$\pm$std\%).}
    \vspace{6mm}
{\fontsize{6.5}{7.5}\selectfont
\begin{tabular}{p{0.95cm}<{\centering} | p{2cm}<{\centering} p{2cm}<{\centering} p{2cm}<{\centering}| p{2cm}<{\centering} p{2cm}<{\centering} p{1.4cm}<{\centering} }
		\toprule%\hline    
            \multirow{2}{*}{\textbf{Method}} &\textbf{RA} &\textbf{FA} &\textbf{TA}  &\textbf{RA} &\textbf{FA} &\textbf{TA} \\
        
        ~ &\multicolumn{3}{c|}{\textbf{400 Samples for 10 Rounds}} & \multicolumn{3}{c}{\textbf{800 Samples for 10 Rounds}} \\ 
        
        \midrule%\hline

        {Retrain}& 72.80$\pm$0.41& 43.51$\pm$0.49& 42.87$\pm$0.26& 72.47$\pm$0.40&42.56$\pm$0.80& 42.16$\pm$0.47\\

        Unroll& \textbf{54.42$\pm$1.63}& 48.96$\pm$1.83& \textbf{40.92$\pm$1.45}& \textbf{56.49$\pm$0.88}&50.41$\pm$0.93& \textbf{42.05$\pm$0.74}\\
        
        CaMU& 39.60$\pm$3.07& 36.02$\pm$5.65& \underline{34.17$\pm$3.09}& 39.01$\pm$1.85&36.34$\pm$2.49& 34.02$\pm$1.50\\

        \rowcolor{gray!20}
        Lcode& 18.07$\pm$6.38& 18.10$\pm$6.42& 12.54$\pm$3.83& 16.63$\pm$6.64&16.14$\pm$6.37& 11.76$\pm$4.04\\
        \rowcolor{gray!20}
        Desc& 27.42$\pm$17.16& 27.68$\pm$17.63& 22.52$\pm$13.70& 27.52$\pm$17.18&26.54$\pm$17.74& 21.86$\pm$13.55\\
        
        \rowcolor{gray!40}
        
        SAFE& \underline{40.44$\pm$5.52}& \textbf{38.26$\pm$3.91}& 33.26$\pm$4.12& \underline{44.63$\pm$1.43}&\textbf{41.47$\pm$1.72}& \underline{37.07$\pm$1.11}\\
        
        \midrule
        ~ &\multicolumn{3}{c|}{\textbf{400 Samples for 40 Rounds}} & \multicolumn{3}{c}{\textbf{800 Samples for 20 Rounds}} \\ 
        \midrule
        {Retrain}& 72.24$\pm$0.60& 42.21$\pm$1.10& 41.60$\pm$0.95& 72.07$\pm$0.54&41.94$\pm$1.08& 40.91$\pm$0.95\\
        
        Unroll& \textbf{54.54$\pm$2.14}& \underline{49.40$\pm$2.04}& \textbf{40.64$\pm$1.45}& \textbf{55.71$\pm$1.66}&50.14$\pm$1.34& \textbf{41.39$\pm$1.24}\\
        
        CaMU& 35.54$\pm$3.97& 33.54$\pm$4.49& 31.20$\pm$3.95& 35.48$\pm$4.16&33.24$\pm$4.14& 31.14$\pm$3.57\\

        \rowcolor{gray!20}
        Lcode& 16.37$\pm$9.41& 16.46$\pm$9.44& 11.27$\pm$5.64& 17.19$\pm$7.63&16.89$\pm$7.54& 11.93$\pm$4.66\\
        \rowcolor{gray!20}
        Desc& 22.98$\pm$9.92& 23.13$\pm$1.18& 19.96$\pm$7.78&24.54$\pm$18.77&23.11$\pm$14.24& 20.13$\pm$12.57\\
        
        \rowcolor{gray!40}
        
        SAFE&\underline{53.62$\pm$2.48}&\textbf{47.54$\pm$2.19}& \underline{43.46$\pm$1.79}& \underline{41.42$\pm$3.49}& \textbf{40.11$\pm$2.47}& \underline{35.43$\pm$2.88} \\
                
        \bottomrule%\hline
        
\end{tabular}\label{settingcomparecifar10feature}
 }
	\vspace{-1mm}
\end{table}

\subsection{Effect Analysis of Learning Rate}\label{appendix:lr}

In this section, we present experiments on hyperparameter tuning, where we evaluate different learning rates $\gamma = \frac{\sqrt{W}}{K\sqrt{T}}$ in the online update algorithm. For the MNIST dataset, we select $K$ values from 1, 2.5, 5, and 10. For the Fashion dataset, we choose $K$ values from 2, 4, 6, and 8. For the CIFAR10 and TinyImageNet datasets, we select $K$ values from 0.4, 1.2, 2, and 4, as well as 0.5, 2.5, 5, and 10, respectively. The plots showing model performance across different rounds are provided in Figure~\ref{results_rounds_lr_mnist}, Figure~\ref{results_rounds_lr_fashion}, Figure~\ref{results_rounds_lr_cifar}, and Figure~\ref{results_rounds_lr_tiny}. 

From the first three figures, we observe that a lower $K$, which implies a higher $\gamma$, often results in more significant changes to the model parameters. This leads to larger degradations in the accuracies of the remaining data, forgetting data, and test data. Conversely, a higher $K$ generally causes smaller changes to the model, resulting in smaller accuracy drops across all three types of data. After considering the balance between the accuracy of the remaining data and the accuracy of the forgetting data, we select the following $K$ values: $K = 2.5$ for MNIST, $K = 8$ for Fashion, $K = 2.5$ for CIFAR10, and $K = 2.5$ for TinyImageNet.

\subsection{Gaussian Distribution Verification}\label{gaussianverify}

Before unlearning, we standardize the low-dimensional vectors such that their mean vector is zero and the covariance matrix is the identity matrix. As each unlearning round progresses, we update the mean vector and covariance matrix. Although the low-dimensional vectors continue to follow a Gaussian distribution, the exact distributions may not be identical across rounds because we removed the vectors of different data in different classes. To verify the Gaussian nature of the vector distribution, we employ Mardia's test, which is highly effective in examining multivariate Gaussian distributions~\cite{mardia1970measures}. The table below presents the Skewness and Kurtosis p-values of Mardia's test [1] on the low-dimensional vectors for each class throughout the unlearning process. The consistently high Skewness and Kurtosis p-values suggest that the low-dimensional vectors maintain a Gaussian distribution across all unlearning rounds.

\begin{table}[H]
\centering
\caption{P-values of skewness and kurtosis tests for different datasets}
\vspace{6mm}
\renewcommand{\arraystretch}{1.2}
\begin{tabular}{lcccc}
\hline
\textbf{Dataset}        & \textbf{Mnist} & \textbf{Fashion} &\textbf{Cifar}   & \textbf{TinyImagenet}\\
\hline
Skewness p\_value & 1.00          & 1.00            & 1.00                    & 1.00          \\
Kurtosis p\_value & 0.93          & 0.95            & 0.84                    & 0.91          \\
\hline
\end{tabular}
\end{table}

\begin{table}[H]\footnotesize
\renewcommand{\arraystretch}{1.2}
        
	\centering
	\caption{Performance comparisons for stream unlearning on a specific class(avg\%$\pm$std\%). SAFE can always achieve the highest average rank in Fashion, CIFAR10, and TinyImagenet. SAFE can achieve \textbf{8 best results} and \textbf{6 second best results} out of all 16 evaluations on different datasets.}
        \vspace{4mm}
        {\fontsize{6.5}{7.5}\selectfont
	\begin{tabular}{p{1.0cm}<{\centering} | p{1.8cm}<{\centering} p{1.8cm}<{\centering} p{1.8cm}<{\centering} p{1.8cm}<{\centering} p{0.6cm}<{\centering} }
		\toprule%\hline    
            \multirow{2}{*}{\textbf{Method}} &\textbf{RA} &\textbf{FA} &\textbf{TA(R)}  &\textbf{TA(F)}&\textbf{R}\\
            
            ~ &\multicolumn{4}{c}{\textbf{MNIST}} \\ 
            
            \midrule%\hline
            \textbf{Retrain}& 99.68$\pm$0.04& 93.34$\pm$21.54& 98.98$\pm$0.21& 93.70$\pm$21.61& -\\
            
            \textbf{Unroll}& 95.08$\pm$1.79& 18.25$\pm$28.72& \underline{97.78$\pm$0.69}& 18.11$\pm$29.06&3.75\\

            \textbf{CaMU}& \textbf{97.04$\pm$0.91}& \textbf{43.72$\pm$30.38}& \textbf{98.55$\pm$0.16}& \textbf{44.13$\pm$30.45}&\textbf{1}\\

            \rowcolor{gray!20}
            \textbf{Lcode}& 32.08$\pm$5.23& 31.30$\pm$9.70& 32.30$\pm$4.77& 30.83$\pm$9.77&4\\
            
            \rowcolor{gray!20}
            \textbf{Desc}& 34.51$\pm$19.40& 30.63$\pm$22.03& 34.84$\pm$19.45& 30.34$\pm$21.99&4\\

            \rowcolor{gray!40}
            \textbf{SAFE}& \underline{96.50$\pm$0.51}& \underline{39.19$\pm$0.24}& 95.96$\pm$0.46& \underline{30.95$\pm$0.10}&\underline{2.25}\\\midrule

            ~ &\multicolumn{4}{c}{\textbf{Fashion}} \\ 
            
            \midrule%\hline
            \textbf{Retrain}& 96.68$\pm$0.39& 71.38$\pm$18.54& 91.61$\pm$0.35& 70.75$\pm$18.58&-\\
            
            \textbf{Unroll}& 84.41$\pm$2.39& \underline{36.84$\pm$4.99}& 86.79$\pm$1.94& \underline{32.30$\pm$4.29}&\underline{2.5}\\

            \textbf{CaMU}& \underline{89.33$\pm$0.10}& 25.41$\pm$2.50& \underline{90.02$\pm$0.45}& 23.56$\pm$2.34&\underline{2.5}\\

            \rowcolor{gray!20}
            \textbf{Lcode}& 22.76$\pm$2.19& 19.27$\pm$5.87& 22.66$\pm$1.94& 18.62$\pm$5.53&5\\
            
            \rowcolor{gray!20}
            \textbf{Desc}& 30.12$\pm$11.09& 19.42$\pm$15.74& 30.45$\pm$11.00& 18.76$\pm$15.10&4\\

            \rowcolor{gray!40}
            \textbf{SAFE}& \textbf{93.36$\pm$0.02}& \textbf{78.62$\pm$0.01}& \textbf{93.02$\pm$0.01}& \textbf{63.37$\pm$0.01}&\textbf{1}\\\midrule

            ~ &\multicolumn{4}{c}{\textbf{CIFAR10}}\\ 
            
            \midrule%\hline
            \textbf{Retrain}& 93.80$\pm$0.20& 73.06$\pm$20.15& 86.23$\pm$0.18& 73.35$\pm$20.27&-\\

            \textbf{Unroll}& 85.73$\pm$2.88& 13.45$\pm$16.93& \textbf{84.07$\pm$2.46}& 13.33$\pm$15.97&3.5\\

            \textbf{CaMU}& \underline{92.84$\pm$1.52}& 52.23$\pm$16.88& 89.11$\pm$0.78& 45.94$\pm$14.44&2.75\\

            \rowcolor{gray!20}
            \textbf{Lcode}& 65.80$\pm$12.43& 21.21$\pm$3.94& 65.47$\pm$12.59& 18.23$\pm$3.41&4.5\\
            
            \rowcolor{gray!20}
            \textbf{Desc}& 84.94$\pm$10.89& \textbf{69.55$\pm$29.53}& 80.79$\pm$27.55&\textbf{65.47$\pm$9.41}&\underline{2.5}\\

            \rowcolor{gray!40}
            \textbf{SAFE}& \textbf{94.30$\pm$0.15}& \underline{59.20$\pm$0.07}& \underline{88.75$\pm$0.13}& \underline{60.10$\pm$0.03}&\textbf{1.75} \\\midrule

            ~ &\multicolumn{4}{c}{\textbf{TinyImagenet}} \\ 
            
            \midrule%\hline
            \textbf{Retrain}& 65.10$\pm$0.08& 37.28$\pm$31.21& 42.91$\pm$0.08& 40.04$\pm$33.03&-\\

            \textbf{Unroll}& 36.25$\pm$1.15& 2.00$\pm$0.04& \underline{27.88$\pm$0.75}& 4.12$\pm$0.08&3.25\\

            \textbf{CaMU}& \underline{37.46$\pm$0.98}& 9.08$\pm$4.73& \underline{27.39$\pm$1.05}& 11.42$\pm$3.69&\underline{2.5}\\

            \rowcolor{gray!20}
            \textbf{Lcode}& 27.90$\pm$16.89& 0.58$\pm$0.13& 26.57$\pm$17.45& \underline{0.53$\pm$0.98}&4.5\\
            
            \rowcolor{gray!20}
            \textbf{Desc}& 13.97$\pm$17.02& \underline{9.11$\pm$9.71}& 14.37$\pm$16.87& 7.31$\pm$7.32&3.75\\

            \rowcolor{gray!40}
            \textbf{SAFE}& \textbf{52.88$\pm$0.11}& \textbf{15.41$\pm$0.14}& \textbf{36.40$\pm$0.06}& \textbf{17.21$\pm$0.01}&\textbf{1}\\

            \bottomrule%\hline
            
	\end{tabular}\label{streamforclass}
        \vspace{-3mm}
        }
\end{table}

\subsection{Results of Streaming Unlearning in Specific Class}

To further evaluate the effectiveness of the SAFE algorithm on single-class data unlearning tasks, we designed experiments that combine stream unlearning in the specific class. In each round, the unlearning request involves removing a subset of data belonging to the same class, continuing until the entire class is removed. We refer to this process as \textit{stream-for-class unlearning}. Experiments were conducted on the MNIST, Fashion, CIFAR10, and TinyImagenet datasets to assess the feasibility of stream-for-class unlearning. In MNIST, Fashion and CIFAR10, we conduct $20$ rounds of unlearning. In each round we remove $300$ data points from class $0$ in MNIST and Fashion and $250$ data points in CIFAR10. In TinyImagenet, $5$ rounds of unlearning are conduct and $80$ data points are removed in each round. By the last round, all data from the target class had been removed, transitioning the task to full-class unlearning.

In these experiments, we primarily evaluate four metrics: \textbf{RA} (training accuracy on the remaining classes), \textbf{UA} (training accuracy on the forgotten class), \textbf{TA(R)} (test accuracy on the remaining classes), and \textbf{TA(U)} (test accuracy on the forgetting class). Similar to our previous experiments, we performed 20 rounds of unlearning, where, by the 20th round, all data from the target class is removed, transitioning the task to full-class unlearning.

Among the five baseline methods, Retrain, Unroll, and CaMU perform single-batch unlearning on all accumulated data for each request, while Lcode, Desc, and SAFE apply stream unlearning for each request. Table~\ref{streamforclass} shows the average results after completing the final round of unlearning. On more complex datasets like, Fashion, CIFAR10, and TinyImagenet, SAFE achieves the best average performance across the four metrics, demonstrating the most balanced trade-off between target class unlearning and preserving knowledge of other classes. In addition, SAFE can achieve $9$ out $12$ best results and $3$ second best ones in such three datasets. This demonstrate the superior performance of SAFE on the unlearning of complex datasets. On the MNIST dataset, CaMU achieves the best results while SAFE can still achieve the second best one. This can be due to the usage of remaining data on CaMU while SAFE does not use any remaining data during the unlearning phases.

\begin{figure*}[htbp]
       \vspace{-4mm}
	\centering
	\hspace{-4mm}
        \subfigure[Remaining data accuracy]{
		\begin{minipage}[t]{0.25\linewidth}
			\centering
			\includegraphics[width=\linewidth]{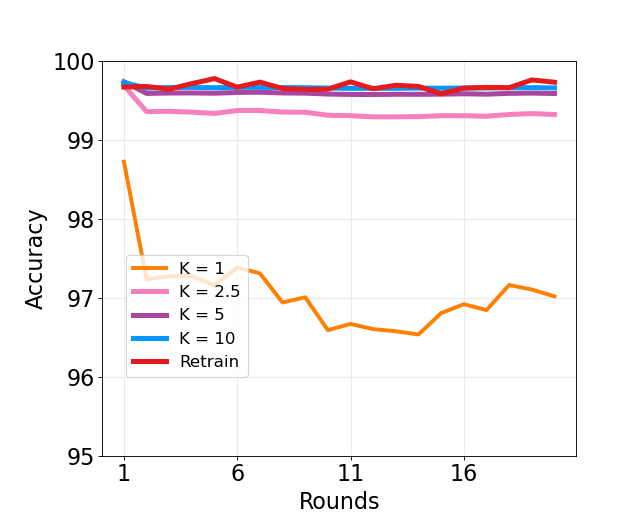}
            \label{rounda_mnist}
		\end{minipage}
	}
        \hspace{2mm}
        \subfigure[Forgetting data accuracy]{
		\begin{minipage}[t]{0.25\linewidth}
			\centering
			\includegraphics[width=\linewidth]{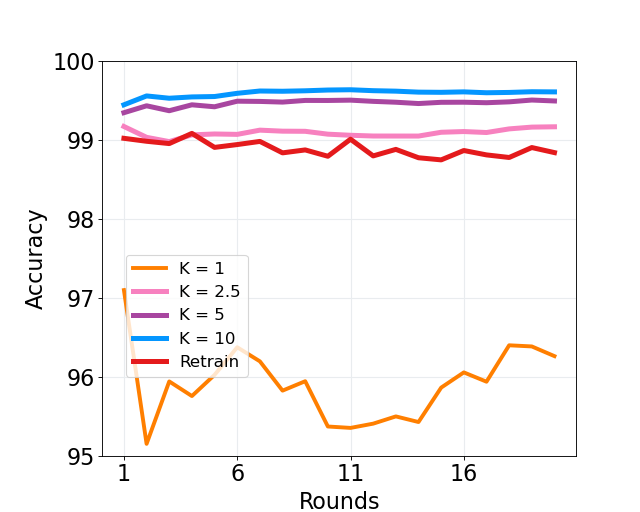}
            \label{roundb_mnist}
		\end{minipage}
	}
        \hspace{2mm}
        \subfigure[Test data accuracy]{
		\begin{minipage}[t]{0.25\linewidth}
			\centering
			\includegraphics[width=\linewidth]{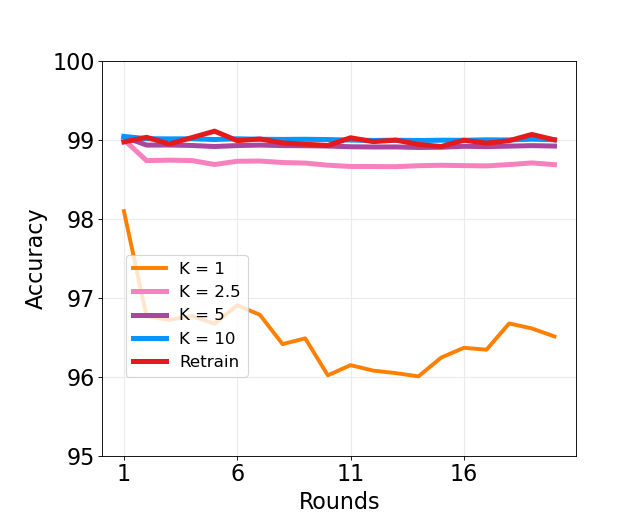}
            \label{roundc_mnist}
		\end{minipage}
	}
        
	\caption{Model performance against unlearning rounds on MNIST.}
	\label{results_rounds_lr_mnist}
\end{figure*}

\begin{figure*}[htbp]
	\centering
	\hspace{-4mm}
        \subfigure[Remaining data accuracy]{
		\begin{minipage}[t]{0.25\linewidth}
			\centering
			\includegraphics[width=\linewidth]{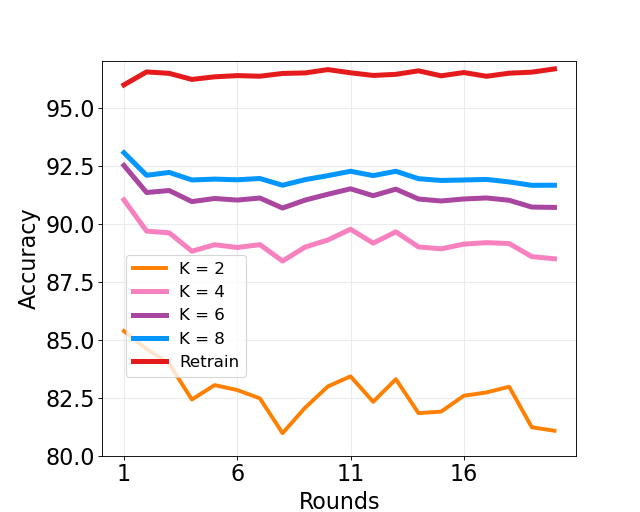}
            \label{rounda_famnist}
		\end{minipage}
	}
        \hspace{2mm}
        \subfigure[Forgetting data accuracy]{
		\begin{minipage}[t]{0.25\linewidth}
			\centering
			\includegraphics[width=\linewidth]{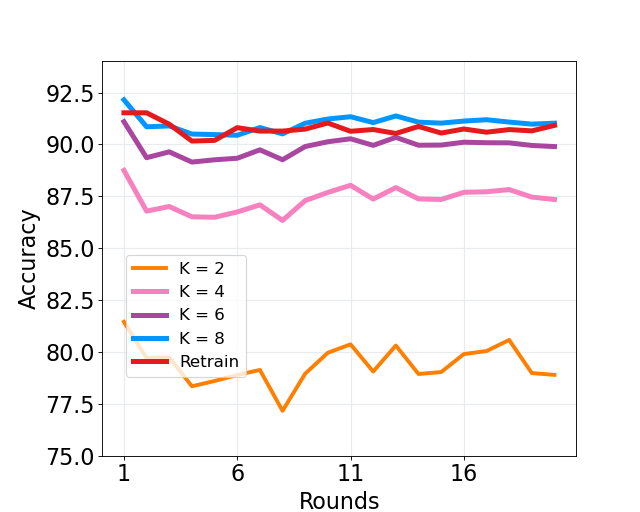}
            \label{roundb_fa}
		\end{minipage}
	}
        \hspace{2mm}
        \subfigure[Test data accuracy]{
		\begin{minipage}[t]{0.25\linewidth}
			\centering
			\includegraphics[width=\linewidth]{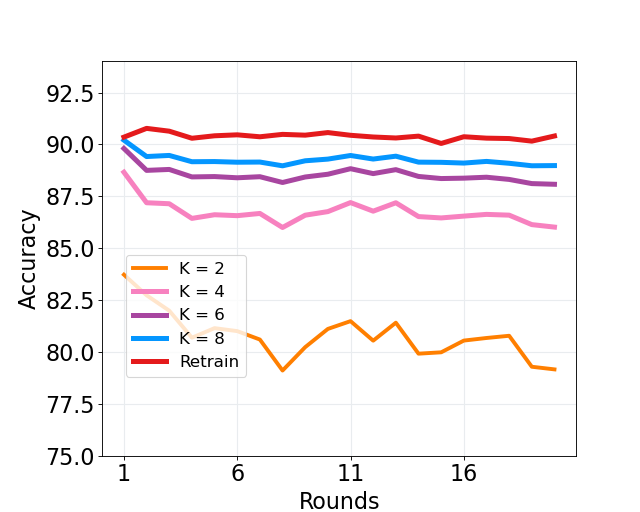}
            \label{roundc_fa}
		\end{minipage}
	}
        
	\caption{Model performance against unlearning rounds on Fashion.}
	\label{results_rounds_lr_fashion}
    \vspace{-4mm}
\end{figure*}

\begin{figure*}[htbp]
	\centering
	\hspace{-4mm}
        \subfigure[Remaining data accuracy]{
		\begin{minipage}[t]{0.25\linewidth}
			\centering
			\includegraphics[width=\linewidth]{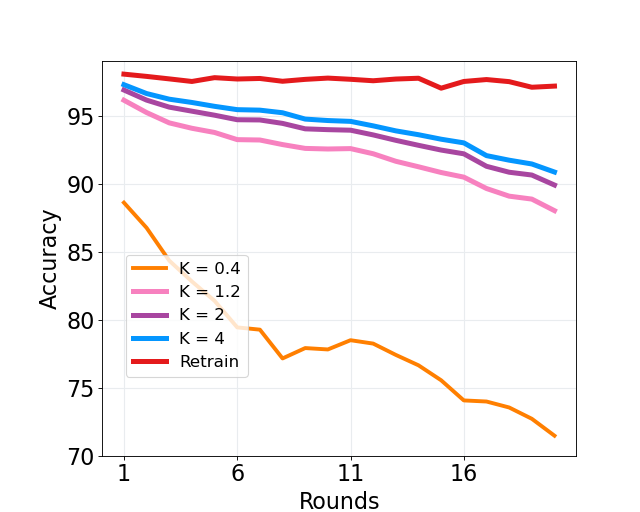}
            \label{rounda_c}
		\end{minipage}
	}
        \hspace{2mm}
        \subfigure[Forgetting data accuracy]{
		\begin{minipage}[t]{0.25\linewidth}
			\centering
			\includegraphics[width=\linewidth]{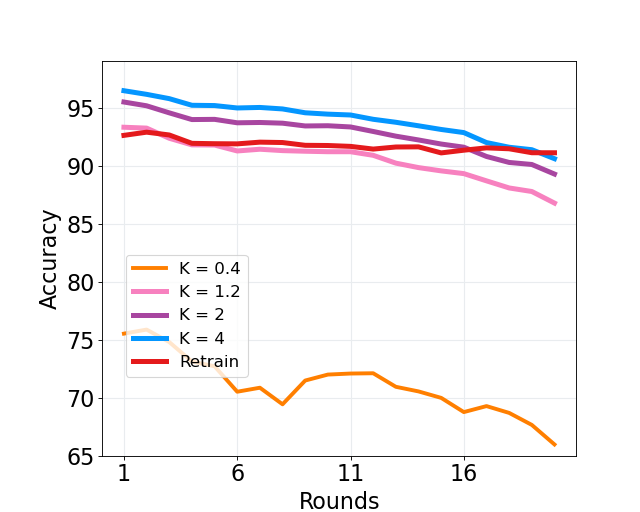}
            \label{roundb_c}
		\end{minipage}
	}
        \hspace{2mm}
        \subfigure[Test data accuracy]{
		\begin{minipage}[t]{0.25\linewidth}
			\centering
			\includegraphics[width=\linewidth]{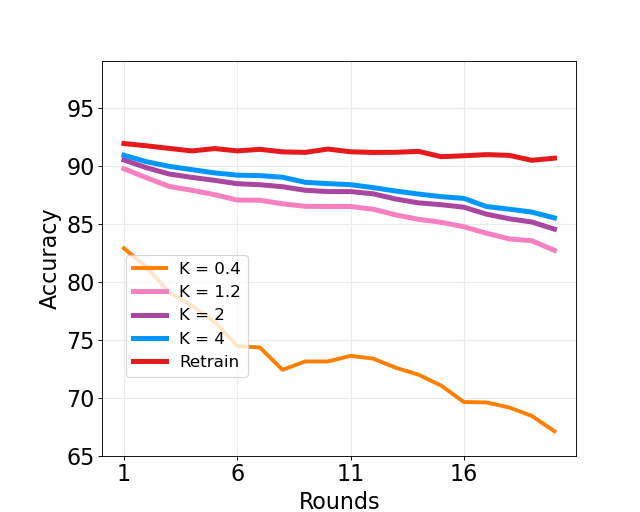}
            \label{roundc_c}
		\end{minipage}
	}
        
	\caption{Model performance against unlearning rounds on CIFAR10.}
	\label{results_rounds_lr_cifar}
    \vspace{-4mm}
\end{figure*}

\begin{figure*}[htbp]
	\centering
	\hspace{-4mm}
        \subfigure[Remaining data accuracy]{
		\begin{minipage}[t]{0.25\linewidth}
			\centering
			\includegraphics[width=\linewidth]{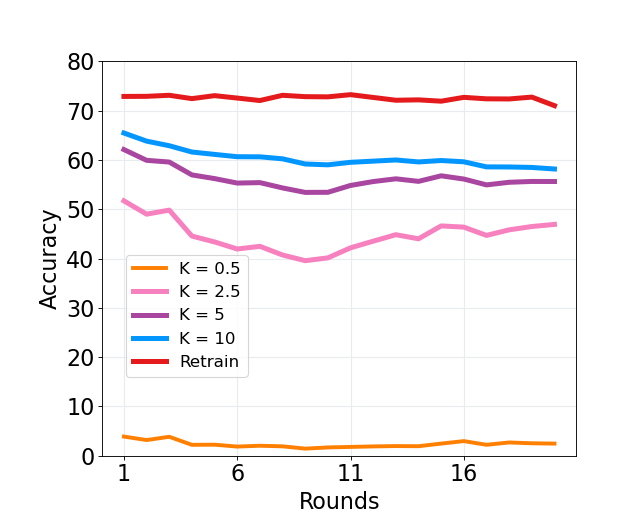}
            \label{rounda_t}
		\end{minipage}
	}
        \hspace{2mm}
        \subfigure[Forgetting data accuracy]{
		\begin{minipage}[t]{0.25\linewidth}
			\centering
			\includegraphics[width=\linewidth]{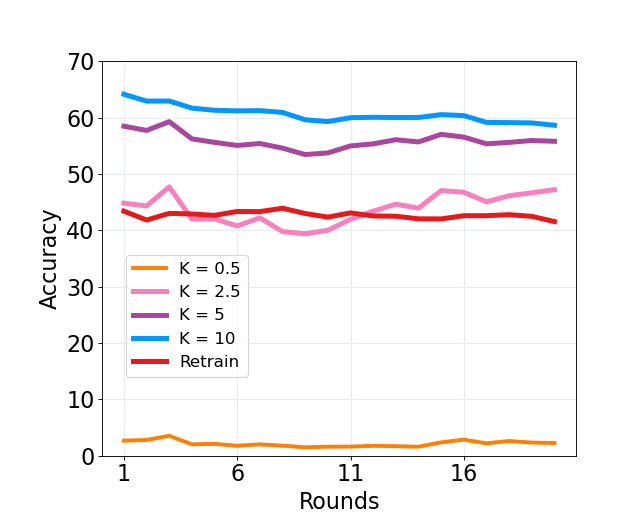}
            \label{roundb_t}
		\end{minipage}
	}
        \hspace{2mm}
        \subfigure[Test data accuracy]{
		\begin{minipage}[t]{0.25\linewidth}
			\centering
			\includegraphics[width=\linewidth]{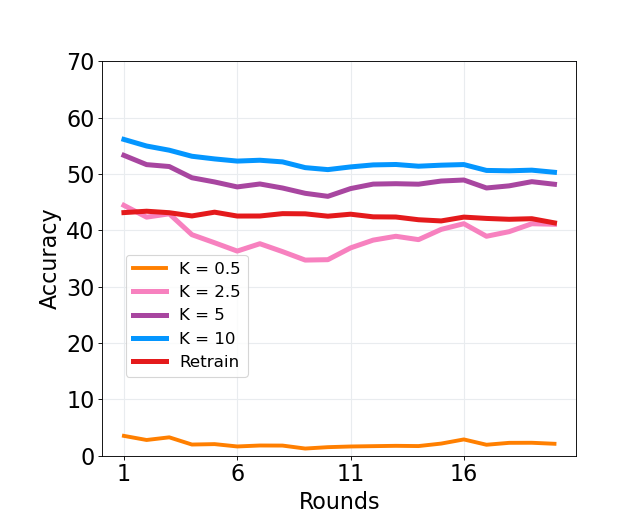}
            \label{roundc_t}
		\end{minipage}
	}
        
	\caption{Model performance against unlearning rounds on TinyImagenet.}
	\label{results_rounds_lr_tiny}
        \vspace{-4mm}
\end{figure*}

\end{document}